\documentclass{article}
\usepackage{PRIMEarxiv}
\usepackage{natbib}
\usepackage{amsthm}
\usepackage{amsmath}
\usepackage{amssymb}
\usepackage{algorithm}
\usepackage{algorithmic}
\usepackage{graphicx}
\usepackage{color}
\usepackage{comment}
\definecolor{myback}{RGB}{204,232,207}

\makeatletter
\newcommand{\AB}{\mathbf{A}}

\newcommand{\FB}{\mathbf{F}}

\newcommand{\PB}{\mathbf{P}}

\newcommand{\aB}{\mathbf{a}}

\newcommand{\fB}{\mathbf{f}}

\newcommand{\hB}{\mathbf{h}}

\newcommand{\uB}{\mathbf{u}}
\newcommand{\vB}{\mathbf{v}}

\newcommand{\xB}{\mathbf{x}}
\newcommand{\yB}{\mathbf{y}}
\newcommand{\zB}{\mathbf{z}}
\newcommand{\AM}{\mathcal{A}}

\newcommand{\OM}{\mathcal{O}}

\newcommand{\RBB}{\mathbb{R}}

\newcommand{\EBB}{\mathbb{E}}

\newcommand{\argmin}{\mathop{\rm argmin}}

\newcommand{\prox}{\mathrm{prox}}
\newcommand{\Time}{{\mathbf{Time}}}
\newcommand{\defi}{\stackrel{\mathrm{def}}{=}}
\newcommand{\lk}{{l_k}}
\newcommand{\ik}{{i_k}}
\makeatother

\newcommand{\clr}[1]{{\color[rgb]{0,0,1} #1}}
\newcommand{\clrg}[1]{{\color[rgb]{0,1,0} #1}}
\newtheorem{theorem}{Theorem}

\newtheorem{lemma}{Lemma}
\newtheorem{definition}{Definition}

\newtheorem{corollary}{Corollary}
\newtheorem{proposition}{Proposition}

\title{Accelerated Doubly Stochastic Gradient Algorithm for Large-scale Empirical Risk Minimization
}
\author{
	Zebang Shen,\quad Hui Qian\thanks{Corresponding author.}, \quad Tongzhou Mu, \quad Chao Zhang\\
	College of Computer Science and Technology \\
	Zhejiang University \\
	\texttt{\{shenzebang, qianhui, mutongzhou, zczju\}@zju.edu.cn}
}

\begin{document}
\maketitle
\begin{abstract}
	Nowadays, algorithms with fast convergence, small memory footprints, and low per-iteration complexity are particularly favorable for artificial intelligence applications.
	In this paper, we propose a doubly stochastic algorithm with a novel accelerating multi-momentum technique to solve large scale empirical risk minimization problem for learning tasks.
	While enjoying a provably superior convergence rate, in each iteration, such algorithm only accesses a mini batch of samples and meanwhile updates a small block of variable coordinates, which substantially reduces the amount of memory reference when both the massive sample size and ultra-high dimensionality are involved.
	Empirical studies on huge scale datasets are conducted to illustrate the efficiency of our method in practice.
	
\end{abstract}
\section{Introduction}
In this paper, we consider the following problem:
\begin{equation} 
\min_{\xB\in\RBB^d} \FB^\PB(\xB) = \FB(\xB) + \PB(\xB),
\label{eqn: problem}
\end{equation}
where $\FB(\xB) = \frac{1}{n}\sum_{i=1}^n f_i(\xB)$ is the average of $n$ convex component functions $f_i$'s and $\PB(\xB)$ is a block-separable and convex regularization function.
We allow both $\FB(\xB)$ and $\PB(\xB)$ to be non-smooth.
In machine learning applications, many tasks can be naturally phrased as the above problem, e.g. the Empirical Risk Minimization (ERM) problem \clr{\cite{zhang2015stochastic,friedman2001elements}}. However, the latest explosive growth of data introduces unprecedented computational challenges of scalability and storage bottleneck. In consequence, algorithms with fast convergence, small memory footprints, and low per-iteration complexity have been ardently pursued in the recent years. Most successful practices adopt stochastic strategies that incorporate randomness into solving procedures. They followed two parallel tracks. That is, in each iteration, gradient is estimated by either a mini batch of \emph{samples} or a small block of \emph{variable coordinates}, commonly designated by a random procedure.

Accessing only a random mini batch of samples in each iteration, classical \emph{Stochastic Gradient Descent} (SGD) and its \emph{Variance Reduction} (VR) variants, such as SVRG~\clr{\cite{johnson2013accelerating}}, SAGA~\clr{\cite{defazio2014saga}}, and SAG~\clr{\cite{schmidt2013minimizing}}, have gained increasing attention in the last quinquennium.
To further improve the performance, a significant amount of efforts have been made towards reincarnating Nesterov's optimal accelerated convergence rate in SGD type methods ~\clr{\cite{zhang2015stochastic,frostig2015regularizing,lin2015universal,shalev2014accelerated,hu2009accelerated,lan2012optimal,nitanda2014stochastic}}. 
Recently, a direct accelerated version of SVRG called Katyusha is proposed to obtain optimal convergence results without compromising the low per-iteration sample access~\clr{\cite{allen2017katyusha,woodworth2016tight}}. However, when datasets are high dimensional, SGD type methods may still suffer from the large memory footprint due to its full vector operation in each iteration, which lefts plenty of scope to push further.

Updating variables only on a randomly selected small block of \emph{variable coordinates} in each iteration is another important strategy that can be adopted solely to reduce the memory reference. 
The most noteworthy endeavors include the \emph{Randomized Block Coordinate Descent} (RBCD) methods ~\clr{\cite{nesterov2012efficiency,richtarik2014iteration,lee2013efficient,wright2015coordinate}}.
Now accelerated versions of RBCD type methods, such as APPROX~\clr{\cite{fercoq2015accelerated}} and APCG~\clr{\cite{lin2015accelerated}}, have also made their debuts. A drawback of RBCD type methods is that all samples have to be accessed in each iteration. When the number of samples is huge, they can be still quite inefficient.
\begin{table*}[t]
	\small
	\centering
	\caption{We give the per-iteration Sample Access (S.A.), Vector Operation (V.O.), and overall computational complexities to obtain an $\epsilon$-accurate solution in relative algorithms. Here APG is short for Accelerated Proximal Gradient. $L$ and $\kappa$ is defined in section 2.1. The mini batch size is set to $1$ for simplicity. $\Omega$ is the block size.}
	\begin{tabular}{|c|c|c|c|c|}
		\hline
		Method  &         S.A.          &          V.O.           &                            General Convex                             &                   Strongly Convex                   \\ \hline
		APG    &       $\OM(n)$        &        $\OM(d)$         &                         $\OM(dn\sqrt{L/\epsilon})$                         &          $\OM(dn\sqrt{\kappa}\log(1/\epsilon))$           \\ \hline
		RBCD   &       $\OM(n)$        &       $\OM(\Omega)$        &                            $\OM(dnL/\epsilon)$                             &             $\OM(dn\kappa\log(1/\epsilon))$             \\ \hline
		APCG   &       $\OM(n)$        &       $\OM(\Omega)$        &                         $\OM(dn\sqrt{L/\epsilon})$                         &          $\OM(dn\sqrt{\kappa}\log(1/\epsilon))$           \\ \hline
		SVRG   &       $\OM(1)$        &        $\OM(d)$         &               $\OM(d(n + L/\epsilon)\log\frac{1}{\epsilon})$               &           $\OM(d(n + \kappa)\log(1/\epsilon))$            \\ \hline
		Katyusha &       $\OM(1)$        &        $\OM(d)$         &                   $\OM(d(n + \sqrt{nL})/\sqrt{\epsilon})$                   &       $\OM(d(n + \sqrt{n\kappa})\log(1/\epsilon))$        \\ \hline
		MRBCD/ASBCD   &       $\OM(1)$        &       $\OM(\Omega)$        &                  $\OM(d(n + L/\epsilon)\log(1/\epsilon))$                  &           $\OM(d(n + \kappa)\log(1/\epsilon))$            \\ \hline
		\textbf{ADSG(this paper)}  & $\OM(1)$ & $\OM(\Omega)$ & $\OM(d(n+\sqrt{nL/\epsilon})\log\frac{1}{\epsilon})$ & $\OM(d(n + \sqrt{n\kappa})\log(1/\epsilon))$ \\ \hline
	\end{tabular}
	\label{table: complexity}
\end{table*}

Doubly stochastic algorithms, simultaneously utilizing the idea of randomness from \emph{sample choosing} and \emph{coordinate selection} perspective, have emerged more recently.
\citeauthor{zhao2014accelerated} carefully combine ideas from VR and RBCD and propose a method called \emph{Mini-batch Randomized Block Coordinate Descent with variance reduction} (MRBCD), which achieves linear convergence in strongly convex case~\clr{\cite{zhao2014accelerated}}.
\citeauthor{zhang2016accelerated} propose the \emph{Accelerated Stochastic Block Coordinate Descent} (ASBCD) method, which incorporates RBCD into SAGA and uses a non-uniform probability when sampling data points \clr{\cite{zhang2016accelerated}}.
These methods avoid both full dataset assess and full vector operation in each iteration, and therefore are amenable to solving (\ref{eqn: problem}) when $n$ and $d$ are large at the same time.
However, none of them meets the optimal convergence rate \clr{\cite{woodworth2016tight}}, and hence can still be accelerated. 

To bridge the gap, we introduce a multi-momentum technique in this paper and devise a method called \emph{Accelerated Doubly Stochastic Gradient algorithm} (ADSG). 
Our method enjoys an accelerated convergence rate, superior to existing doubly stochastic methods, without compromising the low sample access and small per-iteration complexity. Specifically, our contributions are listed as follows.
\begin{enumerate}
	\item With two novel \emph{coupling steps}, we incorporate three momenta into ADSG.
	These two steps enable the acceleration of doubly stochastic optimization procedure. 
	Additionally, we devise an efficient implementation of ADSG for ERM problem.
	\item We prove that ADSG has an accelerated convergence rate and the overall computational complexity to obtain an $\epsilon$-accurate solution is  $\OM((n+\sqrt{n\kappa})\log(1/\epsilon))$ in the strongly convex case, where $\kappa$ stands for the condition number.
	\item Solving general convex and non-smooth problems via reduction.
\end{enumerate}
Further, to show the efficiency of ADSG in practice, we conduct learning tasks on huge datasets with more than 10M samples and 1M features.
The results demonstrate superior computational efficiency of our approach compared to the state-of-the-art.

\section{Preliminary}
	\subsection{Notation \& Assumptions}	
	We assume that the variable $\xB \in \RBB^d$ can be equally partitioned into $B$ blocks for simplicity, and we let $\Omega = d/B$ be the block size.
	We denote the coordinates in the $l^{th}$ block of $\xB$ by $[\xB]_l \in \RBB^\Omega$ and the rest by $[\xB]_{\backslash l}$.
	The regularization $\PB(\xB)$ is assumed to be block separable with respect to the partition of $\xB$, i.e.
	$\PB(\xB) = \sum_{l=1}^B \PB_l([\xB]_l)$.
	Many important functions qualify such separability assumption \clr{\cite{tibshirani1996regression,simon2013sparse}}.
	The proximal operator of a convex function $g$ is defined as 
	$\prox_{g}(\yB) = \argmin_\xB g(\xB) + \frac{1}{2}\|\xB-\yB\|^2$
	where we use $\|\cdot\|$ to denote the Euclidean norm.
	We say $\xB$ is an $\epsilon$-accurate solution of to Problem (\ref{eqn: problem}) if $\FB^\PB(\xB) - \FB^\PB(\xB_*) \leq \epsilon$, where $\xB_*$ is the optimal solution.
	Further, we define $\epsilon_0$ to be $\FB^\PB(\xB_0)$, i.e. the objective value at the initial point $\xB_0$.
	The definition of smoothness, block smoothness, and strong convexity are given as follows.
	\begin{definition}
		A function $f$ is said to be $L$-smooth if for any $\xB, \vB \in\RBB^d$
		\begin{equation}
			f(\xB+\vB) \leq f(\xB) + \langle \nabla f(\xB), \vB\rangle + \frac{L}{2}\|\vB\|^2.
		\end{equation}
	\end{definition}
	\begin{definition}
		A function $f$ is said to be $L_b$ block smooth if for any $l$ and any $\xB, \hB\in \RBB^d$ such that $[\hB]_{\backslash l} \equiv 0$,
		\begin{equation}
			f(\xB + \hB) \leq f(\xB) + \langle \nabla f(\xB), \hB\rangle + \frac{L_b}{2}\|\hB\|^2.
		\end{equation}
	\end{definition}
	From the above two definitions, we have $L_b \leq L$.
	\begin{definition}
		A function $f$ is said to be $\mu$ strongly convex if for any $\xB, \yB\in\RBB^d$ and any $g \in \partial \PB(\xB)$, where $\partial \PB(\xB)$ is the subgradient of $\PB(\cdot)$ at $\xB$,
		\begin{equation}
			\PB(\yB) \geq \PB(\xB) + \langle g, \yB -\xB\rangle + \frac{\mu}{2}\|\xB-\yB\|^2.
		\end{equation}
	\end{definition}
	\noindent We define {\small $\kappa = \frac{L + L_B}{\mu}$} as the condition number.
	\subsection{Doubly Stochastic Algorithms for Convex Composite Optimization}
	A doubly stochastic method that incorporates both VR and CD technique usually consists of two nested loops.
	More specifically, at the beginning of each outer loop, the exact full gradient at some snapshot point $\tilde{\xB}$ is calculated. 
	Then, each iteration of the following inner loop estimates the partial gradient based on a mini-batch component functions and modifies it by the obtained exact gradient to perform block coordinate descent.
	Take MRBCD for example.
	As noted in Table \ref{table: complexity}, its per-iteration vector operation complexity and sample access are $\OM(\Omega)$ and $\OM(1)$ respectively, which is amenable to solving large-sample-high-dimension problems.
	However, its overall computational complexity to achieve an $\epsilon$-accurate solution scales linearly with respect to the condition number $\kappa$ in strongly convex case and depends on $1/\epsilon$ (up to a log factor) in general convex case.
	Such complexity leaves much room to improve in doubly stochastic methods, especially when a highly accurate solution to an ill-conditioned problem is sought (small $\epsilon$, large $\kappa$).
	
\subsection{Momentum Accelerating Technique}
	Accelerating the first order algorithm with momentum is by no mean new but has always been considered difficult \clr{\cite{polyak1964some,AllenOrecchia2017}}.
	Nesterov is the first to prove the accelerated convergence rate for deterministic smooth convex optimization \clr{\cite{nesterov1983method,nesterov1998introductory}}.
	When randomness is involved in data point sampling, noise in the stochastic gradient is further accumulated by the momentum, making the acceleration much harder.
	Efforts have been made to overcome such difficulty in various ways: with an outer-inner loop manner \clr{\cite{lin2015universal}}, from a dual perspective \clr{\cite{shalev2014accelerated}}, or under a primal-dual framework \clr{\cite{zhang2015stochastic}}.
	However, it was not until recently that an optimal method called Katyusha is proposed \clr{\cite{allen2017katyusha,woodworth2016tight}}.
	In a parallel line, at the first attempt to accelerate the RBCD, the momentum step forces full vector operation in each iteration and hence compromises the advantage of RBCD \clr{\cite{nesterov2012efficiency}}.
	With continuing efforts, \citeauthor{lin2015accelerated} proposed the APCG method for both strongly convex and general convex problems, avoiding full vector operation completely \clr{\cite{lin2015accelerated}}.
	While these previous mentioned analyses deal with the randomness only in data point sampling or coordinate choosing, our analysis, due to doubly stochastic nature of ADSG, has to consider how randomness affects the momenta from both sample and feature perspective, and hence is more difficult.

\subsection{Empirical Risk Minimization}
We focus on Empirical Risk Minimization (ERM) with linear predictor, an important class of smooth convex problems.
Specifically, each $f_i$ in Problem (\ref{eqn: problem}) is of the form $f_i(\xB) = \phi_i(\aB_i^\top\xB)$, where $\aB_i$ is the feature vector of the $i^{th}$ sample and $\phi_i(\cdot):\RBB\rightarrow\RBB$ is some smooth convex function.
Let $\AB = [\aB_1 \ldots \aB_n]^\top$ be the data matrix.
In real applications, $\AB$ is usually very sparse and we define its sparsity to be $\rho = \frac{nnz(\AB)}{nd}$.

\subsection{Reduction} \label{section_reduction}
Many existing algorithms work only in restricted settings: SVRG \cite{johnson2013accelerating} and SAGA \cite{defazio2014saga} solves only smooth problems and SDCA \cite{shalev2013stochastic} applies only to strongly convex problems.
To broaden their applicable domain, a common approach is to reduce the target problem to a series of more regular problems, e.g. both smooth and strongly convex, and then call existing methods in a black box manner.
While previous reduction strategies usually have some inevitable drawbacks, e.g. introduce some extra $\log \frac{1}{\epsilon}$ factor to convergence rate, \citeauthor{allen2016optimal} propose three meta algorithms, namely AdaptReg, AdaptSmooth, or JointAdaptRegSmooth, to conduct such reduction procedure efficiently.
However, their strategies do not take doubly stochastic algorithms as input, because such algorithms require the component functions to be block-smooth \cite{zhao2014accelerated,zhang2016accelerated} which is not considered by \cite{allen2016optimal}.
In the following, we extend their idea and define the Homogeneous Objective Decrease (HOOD) property with extra emphasis on the block-smooth parameter $L_B$.
\begin{definition}
	When minimizing an $L$-smooth, $L_B$-block-smooth, and $\mu$-strongly convex function, an algorithm $\AM$ is said to satisfy Homogeneous Objective Decrease (HOOD) property with time Time$(L, L_B, \mu)$ if for every starting point $\xB_0$, it produces output $\xB' \leftarrow \AM(F, \xB_0)$ such that $F(\xB') - \min_xF(x) \leq \frac{F(\xB_0) - \min_xF(x)}{4}$ in $\Time(L, L_B, \mu)$.
\end{definition}
Given a base algorithm satisfies HOOD property, we show the complexity of solving problems with either general convex regularization function or non-smooth component functions via reduction.
\begin{theorem}
	Given an algorithm satisfying HOOD with $\Time(L, L_B, \mu)$ and a starting point $\xB_0$.
	\begin{enumerate}
		\item If each $f_i(\cdot)$ is $L$-smooth and $L_B$-block-smooth, AdaptReg outputs $\xB$ satisfying $\EBB[\FB(\xB)] - \FB(\xB_*)\leq \OM(\epsilon)$ in time 
		\begin{center}
			$\sum_{t=1}^{T-1}$ $\Time(L, L_B, \frac{\mu_0}{2^t})$ where $\mu_0 = \frac{\FB(\xB_0) - \FB(\xB_*)}{\|\xB_0 - \xB_*\|^2}$ and $T = \log_2 \frac{\FB(\xB_0) - \FB(\xB_*)}{\epsilon}$.
		\end{center}
	\end{enumerate}
	\label{thm_AdaptReg}
\end{theorem}
\begin{theorem}
	Given an algorithm satisfying HOOD with $\Time(L, L_B, \mu)$ and a starting point $\xB_0$, and consider the ERM problem where $f_i(\xB) = \phi_i(\aB_i^\top\xB)$.
	\begin{enumerate}
		\item If each $\phi_i(\cdot)$ is $G$-Lipschitz and $\PB(\cdot)$ is $\mu$-strongly-convex, AdaptSmooth outputs $\xB$ satisfying $\EBB[\FB(\xB)] - \FB(\xB_*) \leq \OM(\epsilon)$ in time
		\begin{center}
			$\sum_{t=1}^{T-1}$ $\Time(\frac{2^t}{\lambda_0}\max_i \|\aB_i\|^2, \frac{2^t }{\lambda_0}\max_{i,j}a^2_{i,j}, \mu)$ where $\lambda_0 = \frac{\FB(\xB_0) - \FB(\xB_*)}{G^2}$ and $T = \log_2 \frac{\FB(\xB_0) - \FB(\xB_*)}{\epsilon}$.
		\end{center}
		\item If each $\phi_i(\cdot)$ is $G$-Lipschitz, JointAdaptRegSmooth outputs $\xB$ satisfying $\EBB[\FB(\xB)] - \FB(\xB_*) \leq \OM(\epsilon)$ in time
		\begin{center}
			$\sum_{t=1}^{T-1}$ $\Time(\frac{2^t}{\lambda_0}\max_i \|\aB_i\|^2, \frac{2^t }{\lambda_0}\max_{i,j}a^2_{i,j}, \frac{\mu_0}{2^t})$ where $\lambda_0 = \frac{\FB(\xB_0) - \FB(\xB_*)}{G^2}$, $\mu_0 = \frac{\FB(\xB_0) - \FB(\xB_*)}{\|\xB_0 - \xB_*\|^2}$, and $T = \log_2 \frac{\FB(\xB_0) - \FB(\xB_*)}{\epsilon}$.
		\end{center}
	\end{enumerate}
	\label{thm_JointAdaptRegSmooth}
\end{theorem}
The details of the meta algorithms AdaptReg, AdaptSmooth, and AdaptRegSmooth along with the proofs for Theorem \ref{thm_AdaptReg} and \ref{thm_JointAdaptRegSmooth} are given in the appendix for completeness.

\section{Methodology}
\begin{algorithm}[t]
	\small
	\caption{ADSG I}
	\begin{algorithmic}[1]
		\label{alg: ADSG I}
		\REQUIRE $n, \xB_0, S, b, B, \mu, \{\alpha_{2, s}, \alpha_{3, s}\}_{s=0}^S$
		\STATE $\zB_0 = \tilde{\xB}^0 \leftarrow \xB_0, m \leftarrow Bn, \alpha_{1,s} \leftarrow 1 - \alpha_{2, s} - \alpha_{3, s}$;
		\FOR{$s \leftarrow 0$ \TO $S$}
		\STATE $\bar{L}_s \leftarrow \frac{L}{B\alpha_{3, s}} + L_B$;
		\STATE $ \eta_s \leftarrow \frac{1}{\bar{L}_s\alpha_{2,s}B}, \theta_s \leftarrow 1+\frac{\mu}{\bar{L}_sB^2\alpha_{2, s} + (B-1)\mu}$;
		\STATE $\tilde{\nabla}^s \leftarrow \nabla \FB(\tilde{\xB}^s)$;
		\FOR{$j \leftarrow 1$ \TO $m$}
		\STATE $k \leftarrow sm+j$;
		\STATE $\yB_k \leftarrow \alpha_{1, s} \xB_{k-1} + \alpha_{2,s} \zB_{k-1} + \alpha_{3, s} \tilde{\xB}^s$; \label{eqn: coupling I}
		\STATE sample mini batch $I$ of size $b$ and feature block $l$;
		\STATE $[\vB_k]_l \leftarrow [\tilde{\nabla}^s]_l + \frac{1}{b}\sum_{i\in I} \big([\nabla f_i(\yB_k)]_l - [\nabla f_i(\tilde{\xB}^s)]_l\big)$; \label{eqn: VR gradient}
		\STATE $[\zB_k]_l\!\!\leftarrow\!\! \prox_{\eta_s\PB_l}\!([\zB_{k-1}]_l - \eta_s [\vB_k]_l)$, $[\zB_k]_{\backslash l} \leftarrow [\zB_{k-1}]_{\backslash l}$; \label{eqn: z update}
		\STATE $\xB_k \leftarrow \yB_k + \alpha_{2, s}B(\zB_k - \zB_{k-1})$; \label{eqn: coupling II}
		\ENDFOR
		\STATE sample $\sigma$ from $\{1, \ldots, m\}$ with probability $\frac{\theta_s^{\sigma-1}}{\sum_{i=0}^{m-1}\theta_s^{i-1}}$;\label{eqn_sampling_snapshot}
		\STATE $\tilde{\xB}^{s+1} \leftarrow \xB_{sm+\sigma}$; 
		\ENDFOR
	\end{algorithmic}
\end{algorithm}
We present the proposed ADSG in algorithm \ref{alg: ADSG I}, and discuss about some crucial details in this section.

\noindent{\bf Input:} The input of ADSG varies for strongly convex and general convex problems.
In the former case, $\mu$ should be the strongly convex parameter, and we set {\small $\alpha_{2, s} = \frac{1}{2B}\min \{1, \sqrt{\frac{n}{\kappa}}\}$} and {\small $\alpha_{3, s} = \frac{1}{2B}$}.
In the latter case, $\mu$ is set to $0$, and we set {\small $\alpha_{2, s} = \frac{2}{s + 4B}$} and {\small $\alpha_{3, s} = \frac{1}{2B}$}.
When $\mu = 0$, line \ref{eqn_sampling_snapshot} adopts uniform probability for sampling.

\noindent{\bf Main Body:} Our algorithm is divided into epochs.
Four variables $\xB_k$, $\yB_k$, $\zB_k$, and $\tilde{\xB}^s$ are maintained throughout.
At the beginning of each epoch, the full gradient at the snapshot point $\tilde{\xB}^s$ is computed. 
Updating steps are taken in the follow-up $m$ inner loops, where we randomly select a mini-batch $I$ of size $b$ and a feature block $l$ to construct a mixed stochastic gradient at point $\yB_k$ and perform proximal coordinate descent on the auxiliary variable $\zB_k$. 

\noindent{\bf Momenta:} 
In sharp contrast to existing doubly stochastic algorithms, two \emph{coupling steps} are added in ADSG to accelerate the convergence.
In line \ref{eqn: coupling I}, $\yB_k$ is constructed as the convex combination of $\xB_{k-1}$, $\zB_{k-1}$, and the snapshot point $\tilde{\xB}^s$.
Here, $\zB_{k-1}$ acts as a \emph{historical momentum} that adds weight to the previous stochastic gradient $\{\vB_t\}_{t\leq k}$ (note that $\zB_{k}$ is simply the linear combination of all $\{\vB_t\}_{t\leq k}$ when $\PB \equiv 0$).
This historical momentum is used in many deterministic accelerated methods, e.g. \clr{\cite{beck2009fast}}.
$\tilde{\xB}^s$ serves as a \emph{negative momentum} that ensures the "gradient" variable $\yB_k$ not drifting away from $\tilde{\xB}^s$ and prevents the variance introduced by the randomness from surging.
Such negative momentum is recently proposed in \clr{\cite{allen2017katyusha}}, but with a different weight $\alpha_{3,s} = \frac{1}{2}$.
Note that such weight is crucial to the convergence, see Theorem \ref{thm_strongly_convex}.
In line \ref{eqn: coupling II}, we have $\EBB_{l} \xB_k = \alpha_1 \xB_{k-1} + \alpha_2\zB_{k-1} + \alpha_3\tilde{\xB}^s + \alpha_2(\tilde{\zB}_k-\zB_{k-1})$, where $\tilde{\zB}_k = \prox_{\eta_s\PB}(\zB_{k-1} - \eta_s \vB_k)$.
Hence, $\alpha_{2, s}B(\zB_k - \zB_{k-1})$, in expectation, adds extra weight on the most recent progress $\tilde{\zB}_k - \zB_{k-1}$ and is called \emph{momentum in expectation} for this reason.
These three momenta are the key to the accelerated convergence of ADSG.

\section{Convergence Analysis} \label{section: Convergence Analysis}
In this section, we present the accelerated convergence rate of ADSG when each $f_i$ is smooth and block smooth and $\PB(\cdot)$ is strongly convex.
As a the major contribution of this paper, the proof is given.
With the reduction results given in section \ref{section_reduction}, we then present the overall computational complexities of ADSG when $f_i$ can be non-smooth and $\PB(\cdot)$ can be general convex.
\subsection{Strongly Convex Case}
If $\kappa = \OM(n)$, we can see from Table \ref{table: complexity} that existing doubly stochastic algorithm like MRBCD requires only $\OM(\log1/\epsilon)$ passes over the whole dataset to achieve an $\epsilon$-accurate solution.
However, facing ill-conditioned problems where $\kappa > n^2$, the performance of such method decays faster than the deterministic APG method due to its linear dependence on the condition number $\kappa$.
The following theorem shows that ADSG enjoys an accelerated convergence rate and depends only on $\sqrt{\kappa}$.
\begin{theorem}
	Suppose Assumption I-III are satisfied. 
	Set {\small $\alpha_{2, s} = \frac{1}{2B}\min \{1, \sqrt{\frac{n}{\kappa}}\}, \alpha_{3, s} = \frac{1}{2B}$} and set the mini batch size to be $1$. If $\kappa > 8B$, for $s\geq0$, we have
	{\small \[\EBB\FB^{\PB}(\tilde{\xB}^s) - \FB^{\PB}(\xB_*\!)\!\leq\!\OM(1)\min\{\frac{9}{8},\!(1\!+\!\sqrt{\frac{n}{\kappa}})\}^{-s}\!\big(\FB^{\PB}(\xB_0) - \FB^{\PB}\!(\xB_*\!)\big),\]}
	and therefore ADSG takes $\OM((1+\sqrt{\frac{\kappa}{n}})\log\frac{\epsilon_0}{\epsilon})$ outer loops to achieve an $\epsilon$-accurate solution.
	\label{thm_strongly_convex}
\end{theorem}

\subsection{Proof for Strongly Convex Case}
The idea of the first lemma is to express $\xB_k$ as the convex combination of $\{\tilde{\xB}^i\}_{i=0}^s$ and $\{\zB_l\}_{l=0}^k$.
\begin{lemma}
	In Algorithm I, by setting $\alpha_{2, 0} = \alpha_{3, 0} = 1/2B$, for $k = sm+j\geq 1$, we have
	\begin{equation}
		\xB_k = \sum_{i=0}^{s-1}\lambda_k^i \tilde{\xB}^i + \beta_j^s \tilde{\xB}^s + \sum_{l=0}^{k} \gamma_k^l \zB_l,
	\end{equation}
	where $\gamma_0^0 = 1$, $\gamma_1^0 = \frac{1}{2} - \frac{1}{2B}$, $\gamma_1^1 = \frac{1}{2}$, $\beta^0_0 = 0$, $\beta^s_0 = \alpha_{3, s}$, $\lambda_{(s+1)m}^s = \beta_m^s$, $\lambda^i_{k+1} = \alpha_{1,s}\lambda^i_k,$
	\begin{equation}
	\gamma_{k+1}^l = 
	\begin{cases}
	\alpha_{1,s}\gamma_k^l,~&l=0,\ldots,k-1,\\
	B\alpha_{1,s}\alpha_{2,s}+(1-B)\alpha_{2,s},~&l=k, \\
	B\alpha_{2,s},~&l=k+1,
	\end{cases}
	\end{equation}
	and
	\begin{equation}
	\beta_{j+1}^s = \alpha_{1,s}\beta_j^s+\alpha_{3,s}.
	\end{equation}
	Additionally, we have $\sum_{i=0}^{s-1}\lambda_k^i + \beta_j^s + \sum_{l=0}^k \gamma_k^l = 1$.
	If all $\alpha_{1, s} \geq \frac{B-1}{B}$, then each entry in this sum is non-negative for all $k\geq1$, i.e. $\xB_k$ is a convex combination of $\{\tilde{\xB}^i\}_{i=0}^s$ and $\{\zB_l\}_{l=0}^k$.
	\label{lemma_convex_combination}
\end{lemma}
\begin{proof}
	We prove by induction.
	When $s=0$, 
	\begin{equation*}
		\begin{aligned}
			\xB_0 &= \zB_0 \\
			\yB_1 &= \alpha_{1, 0} \zB_0 + \alpha_{2, 0}\zB_0 + \alpha_{3, 0}\tilde{\xB}^0\\
			\xB_1 &= (\alpha_{1, 0} +\alpha_{2, 0})\zB_0 + B\alpha_{2, 0}(\zB_1 - \zB_0) + \alpha_{3, 0}\tilde{\xB}^0\\
			& = (\frac{1}{2} - \frac{1}{2B})\zB_0 + \frac{1}{2}\zB_1 + \frac{1}{2B}\tilde{\xB}^0
		\end{aligned}
	\end{equation*}
	which proves the initialization.
	Assume that our formulation is correct up till the $\kappa^{th}$ iteration.
	In the $(\kappa+1)^{th}$ iterations,
	\begin{equation*}
		\begin{aligned}
			\yB_{\kappa+1} =~& \alpha_{1, s} \xB_\kappa + \alpha_{2, s}\zB_\kappa + \alpha_{3, s}\tilde{\xB}^s\\
			\xB_{\kappa+1} =~& \alpha_{1, s} \xB_\kappa + \alpha_{2, s}\zB_\kappa + \alpha_{3, s}\tilde{\xB}^s + \alpha_{2, s}B(\zB_{\kappa+1} - \zB_\kappa)\\
			=~& \underbrace{\alpha_{1, s} (\sum_{l=1}^{\kappa -1} \gamma_\kappa^l \zB_l)}_{\gamma_{\kappa+1}^l, l=0,\ldots,\kappa-1}
			+ \underbrace{(\alpha_{1,s}B\alpha_{2,s}+(1-B)\alpha_{2,s})\zB_\kappa}_{\gamma_{\kappa+1}^\kappa}
			+ \underbrace{B\alpha_{2,s}\zB_{\kappa+1}}_{\gamma_{\kappa+1}^{\kappa+1}}
			+ \underbrace{(\alpha_{1,s}\beta_j^s+\alpha_{3,s})\tilde{\xB}^s}_{\beta_{j+1}^s}
			+ \underbrace{\alpha_{1,s}\sum_{i=0}^{s-1} \lambda_\kappa^i\tilde{\xB}^i}_{\lambda_{\kappa+1}^i, i=0,\ldots,s-1}
		\end{aligned}
	\end{equation*}
	which gives us the results about $\{\gamma_{\kappa+1}^l\}_{l=0}^{\kappa+1}$, $\beta_{j+1}^s$, and $\{\lambda_{\kappa+1}^i\}_{i=0}^{s-1}$.
	In the beginning of the $s^{th}$ epoch, we have $\beta_0^s = \alpha_{3, s}$ due to the constructions of $\yB_k$ and $\xB_k$ in the algorithm. Since $\tilde{\xB}^i$ is only added after the $i^{th}$ epoch is done, it is initialized as $\lambda^i_{k+1} = \alpha_{1,s}\lambda^i_k = \beta_j^s$.
\end{proof}
The second lemma analyzes ADSG in one iteration.
Before proceeding, we first define a few terms:
(1) {\small $\hat{\PB}(\xB_k) \defi \sum_{i=0}^{s-1}\lambda_k^i \PB(\tilde{\xB}^i) + \beta_j^s \PB(\tilde{\xB}^s) + \sum_{l=0}^{k} \gamma_k^l \PB(\zB_l) \geq \PB(\xB_k)$}, where the inequality uses the convexity of $\PB$;
(2) {\small $d(\xB_k) \defi \FB^\PB(\xB_k) - \FB^\PB(\xB_*)$};
(3) {\small $\tilde{d}^s\defi \FB^\PB(\tilde{\xB}^s) - \FB^\PB(\xB_*)$};
and (4) {\small $d_k \defi (\FB(\xB_k) + \hat{\PB}(\xB_k)) - \FB^\PB(\xB_*)$}.
Additionally, we have $0\leq d(\xB_k) \leq d_k$ and $d(\xB_0) = d_0$.
\begin{lemma}
	In the $s^{th}$ epoch of Algorithm \ref{alg: ADSG I}, we have
	\begin{equation}
		\begin{aligned}
			\EBB_{i,l} \hat{d}_{j} &+ \theta\EBB_{i,l}A_j \leq \alpha_3\tilde{d}^s + \alpha_1\hat{d}_{j-1} + A_{j-1},
		\end{aligned}
		\label{eqn_lemma_key}
	\end{equation}
	with $\rho = \alpha_2^2B^2\bar{L} + (B-1)\mu\alpha_2$, $\theta = 1 + \mu\alpha_2/\rho$, $A_j = \rho\|\xB^* - \zB_{sm+j}\|^2/2$, and $\hat{d}_{j} = d_{sm+j}$.
	\label{lemma_key}
\end{lemma}
\begin{proof}
	Define $\tilde{\zB}_k = \prox_{\eta\PB}(\zB_{k-1} - \eta v_k)$.
	We have $[\zB_k]_l = [\tilde{\zB}_k]_l$ if the $l^{th}$ block is selected in the $k^{th}$ iteration and $[\zB_k]_l = [\zB_{k-1}]_l$ otherwise.
	From the construction of $\xB_k$, we have $\xB_k - \yB_k = B(\alpha_1 \xB_{k-1} + \alpha_2\zB_k + \alpha_3\tilde{\xB}^s - \yB_k)$.
	In particular, $[\xB_k - \yB_k]_l = B(\alpha_1 [\xB_{k-1}]_l + \alpha_2[\tilde{\zB}_k]_l + \alpha_3[\tilde{\xB}^s]_l - [\yB_k]_l)$.
	Using Assumption III, we have
	\begin{equation*}
		\begin{aligned}
		F(\xB_k) \leq&~ F(\yB_k) + \langle [\nabla F(\yB_k)]_l, [\xB_k - \yB_k]_l \rangle + \frac{L_l}{2}\|[\xB_k - \yB_k]_l\|^2 \\
		=&~ F(\yB_k) + \langle [v_k]_l, [\xB_k - \yB_k]_l \rangle + \langle [\nabla F(\yB_k)]_l - [v_k]_l, [\xB_k - \yB_k]_l \rangle + \frac{L_l}{2}\|[\xB_k - \yB_k]_l\|^2 \\
		\leq&~ F(\yB_k) + \frac{B\alpha_3}{2L}\|[\nabla F(\yB_k)]_l - [v_k]_l\|^2 + \frac{\bar{L}}{2}\|[\xB_k - \yB_k]_l\|^2 \\
		&+ \langle [v_k]_l, B\alpha_1[\xB_{k-1} - \yB_k]_l\rangle + \langle [v_k]_l, B\alpha_2[\tilde{\zB}_k - \yB_k]_l\rangle + \langle [v_k]_l, B\alpha_3[\tilde{\xB}^s - \yB_k]_l\rangle,
		\end{aligned}
	\end{equation*}
	where we use Young's inequality, i.e. $\langle \vec{a}, \vec{b} \rangle \leq c \|\vec{a}\|^2/2 + \|\vec{b}\|^2/2c$, and $\bar{L} = L/B\alpha_3 + L_l$.
	Taking expectation with respect to the block coordinate random variable $l$, we have
	\begin{equation}
		\begin{aligned}
			\EBB_l F(\xB_k) \leq F(\yB_k) &+ \frac{\alpha_3}{2L}\|\nabla F(\yB_k) - v_k\|^2 +  \frac{\bar{L}B\alpha_2^2}{2}\|\tilde{\zB}_k - \zB_{k-1}\|^2 \\
			&+ \alpha_1\langle v_k, \xB_{k-1} - \yB_k\rangle + \alpha_2\langle v_k, \tilde{\zB}_k - \yB_k\rangle + \alpha_3\langle v_k, \tilde{\xB}^s - \yB_k\rangle
		\end{aligned}
		\label{eqn_proof_I}
	\end{equation}
	
	We define $\hat{\PB}(\xB_k) = \sum_{i=0}^{s-1}\lambda_k^i \PB(\tilde{\xB}^i) + \beta_j^s\PB(\tilde{\xB}) + \sum_{l=0}^{k} \gamma_k^l\PB(\zB_l)$.
	From Lemma \ref{lemma_convex_combination} and the convexity of $\PB(\cdot)$, we have $\PB(\xB_k) \leq \hat{\PB}(\xB_k)$.
	Taking expectation with respect to the block coordinate random variable $l_k$, we have
	\begin{equation}
		\begin{aligned}
			\EBB_{l_k} \hat{\PB}(\xB_k) =&  \sum_{i=0}^{s-1}\lambda_k^i \PB(\tilde{\xB}^i) + \beta_j^s\PB(\tilde{\xB}) + \sum_{l=0}^{k-1} \gamma_k^l\PB(\zB_l) + B\alpha_2\EBB_{l_k} \PB(\zB_k)\\
			=& \sum_{i=0}^{s-1}\lambda_k^i \PB(\tilde{\xB}^i)+ \beta_j^s\PB(\tilde{\xB}) + \sum_{l=0}^{k-1} \gamma_k^l\PB(\zB_l) + \alpha_2(B-1)\PB(\zB_{k-1})+\alpha_2\PB(\tilde{\zB}_k)\\
			=& \alpha_1 \hat{\PB}(\xB_{k-1}) + \alpha_2 \PB(\tilde{\zB}_k) + \alpha_3 \PB(\tilde{\xB}^s).
		\end{aligned}
		\label{eqn_proof_II}
	\end{equation}
	
	Add (\ref{eqn_proof_I}) and (\ref{eqn_proof_II}) and take expectation with respect to the sample random variable $i_k$.
	Using the unbiasedness of $v_k$, i.e. $\EBB_{i_k} v_k = F(\yB_k)$, we have
	\begin{equation}
		\begin{aligned}
			\EBB_{i_k, l_k} F(\xB_k) + \hat{\PB}(\xB_k) \leq&~ F(\yB_k) + \alpha_3(\frac{1}{2L}\|\nabla F(\yB_k) - v_k\|^2 + \langle \nabla F(\yB_k), \tilde{\xB}^s - \yB_k\rangle + \PB(\tilde{\xB}^s))\\
			& + \alpha_2(\frac{\bar{L}B\alpha_2}{2}\|\tilde{\zB}_k - \zB_{k-1}\|^2 + \langle v_k, \tilde{\zB}_k - \yB_k\rangle + \PB(\tilde{\zB}_k))\\
			& + \alpha_1(\langle \nabla F(\yB_k), \xB_{k-1} - \yB_k\rangle + \hat{\PB}(\xB_{k-1})). \\
		\end{aligned}
	\end{equation}
	For any random variable $a$, we have $\EBB\|a - \EBB a\|^2 \leq \EBB\|a\|^2$. 
	Setting $a = \nabla f_i(\yB_k) - \nabla f_i(\tilde{\xB}^s)$, we have 
	$\EBB_{i_k}\|\nabla F(\yB_k) - v_k\|^2 \leq \EBB_{i_k}\|\nabla f_i(\yB_k) - \nabla f_i(\tilde{\xB}^s)\|^2$.
	From the smoothness of $f_i(\cdot)$, we have $\EBB_{i_k}[\frac{1}{2L}\|\nabla F(\yB_k) - v_k\|^2 + \langle v_k, \tilde{\xB}^s - \yB_k\rangle + F(\yB_k)] \leq F(\tilde{\xB}^s)$. Additionally, with the convexity of $F(\cdot)$, we have
	\begin{equation}
		\begin{aligned}
			\EBB_{i_k, l_k} F(\xB_k) + \hat{\PB}(\xB_k) \leq&~ \alpha_1 (F(\xB_{k-1}) + \hat{\PB}(\xB_{k-1})) + \alpha_3 (F(\tilde{\xB}^s) + \PB(\tilde{\xB}^s)) \\
			& + \alpha_2 (F(\yB_k) + \underbrace{\frac{\bar{L}B\alpha_2}{2}\|\tilde{\zB}_k - \zB_{k-1}\|^2 + \langle v_k, \tilde{\zB}_k - \yB_k\rangle + \PB(\tilde{\zB}_k)}_{h(\tilde{\zB}_k)}).
		\end{aligned}
	\end{equation}
	Due to the construction of $\tilde{\zB}_k = \prox_{\eta\PB}(\zB_{k-1} - \eta v_k) = \argmin_\zB h(\zB)$ and the $\mu$-strong convexity of $h(\cdot)$, we have $h(\tilde{\zB}_k) + \mu/2 \cdot \|\zB^* - \tilde{\zB}_k\|^2 \leq h(\zB^*)$ and hence
	\begin{equation*}
		\begin{aligned}
			\EBB&_{i_k, l_k} F(\xB_k) + \hat{\PB}(\xB_k) \leq~ \alpha_1 (F(\xB_{k-1}) + \hat{\PB}(\xB_{k-1})) + \alpha_3 (F(\tilde{\xB}^s) + \PB(\tilde{\xB}^s)) \\
			& + \alpha_2(F(\yB_k) + \langle\nabla F(\yB_k), \xB^* - \yB_k\rangle +\PB(\xB^*) + \frac{\bar{L}\alpha_2B}{2}\|\xB^* - \zB_{k-1}\|^2 - \frac{\bar{L}\alpha_2B + \mu}{2}\EBB_{i_k}\|\xB^* - \tilde{\zB}_k\|^2).
		\end{aligned}
	\end{equation*}
	From the convexity of $F(\cdot)$ and $\EBB_l[\|\zB_k - \xB\|^2] =  \frac{1}{B}\|\tilde{\zB}_k - \xB\|^2 + \frac{B-1}{B}\|\zB_{k-1}-\xB\|^2$, we have
	\begin{equation*}
		\begin{aligned}
			\EBB_{i_k, l_k} F(\xB_k) + \hat{\PB}(\xB_k) \leq&~ \alpha_1 (F(\xB_{k-1}) + \hat{\PB}(\xB_{k-1})) + \alpha_3 (F(\tilde{\xB}^s) + \PB(\tilde{\xB}^s)) + \alpha_2(F(\xB^*) + \PB(\xB^*)) \\
			& \frac{\alpha_2^2B + \frac{B-1}{B}\frac{\mu\alpha_2}{\bar{L}_s}}{2}B\bar{L}_s\|\xB^* - \zB_{k-1}\|^2 - \frac{\alpha_2^2B + \frac{\mu\alpha_2}{\bar{L}_s}}{2}B\bar{L}_s\EBB_{l, i_k}\|\xB^* - \zB_k\|^2.
		\end{aligned}
	\end{equation*}
	Subtract $F^\PB(\xB^*)$ from both sides and we have the lemma.
\end{proof}
With Lemma \ref{lemma_key}, we give the proof for Theorem \ref{thm_strongly_convex} as follows.
\begin{proof}
	We omit the expectation for simplicity.
	By multiplying $\theta^j$ to both sides of (\ref{eqn_lemma_key}), summing from $j = 1$ to $m-1$, and rearranging terms, we have
	\begin{equation}
		(1 - \theta\alpha_1)\!\!\sum_{j=1}^m\theta^{j-1}\hat{d}_j + \alpha_1\theta^m\hat{d}_m + \theta^m A_m \!\leq\! \alpha_1\hat{d}_0 + \alpha_3\tilde{d}^s\sum_{j=0}^{m-1}\theta^{j} + A_0.
	\end{equation}
	Using the definition of $\tilde{\xB}^{s+1}$, $\hat{d}_j \geq d(\xB_{sm+j})$, and the convexity of $\FB^\PB$, we have 
	\begin{equation}
		(1 - \theta\alpha_1)\tilde{d}^{s+1}\!\!\sum_{j=0}^{m-1}\theta^j + \alpha_1\theta^m\hat{d}_m + \theta^m A_m \! \leq\! \alpha_1\hat{d}_0 + \alpha_3\tilde{d}^s\!\!\sum_{j=0}^{m-1}\theta^{j} + A_0.
	\end{equation}
	\noindent {\bf Case 1 ($\kappa > 2n$)}:
	From the settings of $\alpha_2$, $\alpha_3, m$, and $\kappa > 2n$, we have 
	{\small
	$\alpha_3(\theta^{m-1} - 1) + (1 - \frac{1}{\theta}) 
	\leq \frac{n\mu}{\rho} 
	\leq \frac{n\mu}{\bar{L} B^2\alpha_2} 
	= \alpha_2.$}
	Therefore we have $1 - \alpha_1\theta \geq \alpha_3 \theta^m$ and
	\begin{equation}
		\theta^m(\alpha_3\tilde{d}^{s+1}\sum_{j=0}^{m-1}\theta^j + \alpha_1\hat{d}_m + A_m) \leq \alpha_1\hat{d}_0 + \alpha_3\tilde{d}^s\sum_{j=0}^{m-1}\theta^{j} + A_0.
	\end{equation}
	Additionally, {\small $\theta \geq 1 + \frac{1}{2B}\sqrt{\frac{\mu}{n\bar{L}}}$}, because {\small $\bar{L}B^2\alpha_2 \geq (B-1)\mu$}, and therefore {\small $\theta^{-m} \leq exp(-4\sqrt{\frac{n\mu}{\bar{L}}})$}.
	Besides, {\small $\alpha_3\sum_{j=0}^{m-1}\theta^j = \frac{1}{2B} \frac{\theta^m - 1}{\theta - 1} \geq \frac{1}{2B}m = n/2$}, thus we have {\small $\frac{\alpha_1\hat{d}_0}{\alpha_3\sum_{j=0}^{m-1}\theta^j} \leq \frac{2}{n}\hat{d}_0$} and {\small $\frac{A_0}{\alpha_3\sum_{j=0}^{m-1}\theta^j} \leq n\mu \frac{2}{n}\|\zB_0 - \xB_*\|^2 = 2\mu\|\zB_0 - \xB_*\|^2$}.
	Using the $\mu$-strongly convexity of $\FB^\PB$, we obtain {\small $\tilde{d}^S \leq \OM(1)exp(-S\sqrt{\frac{n\mu}{\bar{L}}})(\hat{d}_0 + \mu\|\xB_0 - \xB_*\|^2) \leq \OM(1)exp(-S\sqrt{\frac{n\mu}{\bar{L}}})\hat{d}_0$}.
	
	\noindent {\bf Case 2 ($\kappa \leq 2n$)}:
	Since $\bar{L} \geq \mu$, we have $\rho = \alpha_2^2B^2\bar{L} + (B-1)\mu\alpha_2 = \bar{L}/4 + \mu\cdot(B-1)/2B \leq 3\bar{L}/4$, and hence $\theta^m \geq 1 + \mu\alpha_2/\rho\cdot Bn \geq 4/3$.
	Further, $\rho \geq \alpha_2^2B^2\bar{L} = \bar{L}/4$ and hence $(1 - \theta\alpha_1)/\alpha_3 \geq 4/3$ given that $\kappa \geq 6$.
	Consequently, we have
	\begin{equation}
		(\frac{4}{3})^m(\alpha_3\tilde{d}^{s+1}\sum_{j=0}^{m-1}\theta^j + \alpha_1\hat{d}_m + A_m) \leq \alpha_1\hat{d}_0 + \alpha_3\tilde{d}^s\sum_{j=0}^{m-1}\theta^{j} + A_0.
	\end{equation}
	Using the same derivation as Case 1, we have {\small $\tilde{d}^S \leq \OM(1)(\frac{4}{3})^{-S}(\hat{d}_0 + \mu\|\xB_0 - \xB_*\|^2) \leq \OM(1)(\frac{4}{3})^{-S}\hat{d}_0$}.
\end{proof}
\subsection{Solving Smooth General Convex Problem}
Theorem \ref{thm_strongly_convex} shows that ADSG satisfies HOOD property with $\Time(L, L_b, \mu) = \OM(n + \sqrt{n(L+L_b)/\mu})$.
By applying Theorem \ref{thm_AdaptReg} and \ref{thm_JointAdaptRegSmooth}, we have the following corollaries.
\begin{corollary}
	If each $f_i(\cdot)$ is $L$-smooth and $L_B$-block-smooth and $\PB(\cdot)$ is general convex in Problem \ref{eqn: problem}, then by applying AdaptReg on ADSG with a starting vector $\xB_0$, we obtain an output $\xB$ satisfying $F^\PB(\xB) - F^\PB(\xB^*) \leq \epsilon$ with computational complexity at most
	\begin{equation*}
	\OM(n\log(\epsilon_0/\epsilon) + \sqrt{n\bar{L}}\|\xB_0 - \xB^*\|/\sqrt{\epsilon}).
	\end{equation*}
\end{corollary}

\begin{corollary}
	If each $f_i(\cdot)$ is $G$-Lipschitz continuous and $\PB(\cdot)$ is $\mu$-strongly convex, then by applying AdaptSmooth on ADSG with a starting vector $\xB_0$ , we obtain an output $\xB$ satisfying $\EBB F^\PB(\xB) - F^\PB(\xB^*) \leq \epsilon$ with computational complexity at most
	\begin{equation*}
	\OM(n\log(\epsilon_0/\epsilon) + \sqrt{n}G\max_i \|\aB_i\|/\sqrt{\mu\epsilon}).
	\end{equation*}
\end{corollary}

\begin{corollary}
	If each $f_i(\cdot)$ is $G$-Lipschitz continuous and $\PB(\cdot)$ is general convex, then by applying JointAdaptRegSmooth on ADSG with a starting vector $\xB_0$ , we obtain an output $\xB$ satisfying $\EBB F^\PB(\xB) - F^\PB(\xB^*) \leq \epsilon$ with computational complexity at most
	\begin{equation*}
	\OM(n\log(\epsilon_0/\epsilon) + \sqrt{n}G\max_i \|\aB_i\|\|\xB_0 - \xB^*\|/\epsilon).
	\end{equation*}
\end{corollary}

\section{Efficient Implementation}
\begin{algorithm}[t]
	\small
	\caption{ADSG II}
	\begin{algorithmic}[1]
		\label{alg: ADSG II}
		\REQUIRE $n, \xB_0, S, b, B, \mu, \{\alpha_{2, s}, \alpha_{3, s}\}_{s=0}^S$
		\STATE $\hat{\zB}_0^0 \!\leftarrow\!\vec{0}, \xi\!= \!\dot{\xB}^0 \!\leftarrow\! \xB_0, k\! \leftarrow\! 0, m \!\leftarrow\! Bn, \alpha_{1,s}\!\! \leftarrow \!\!1 \!-\! \alpha_{2, s}\! - \!\alpha_{3, s}$;
		\FOR{$s \leftarrow 0$ \TO $S$} 
		\STATE $\gamma_s \leftarrow \frac{\alpha_{2,s}}{\alpha_{2,s} + \alpha_{3,s}}, \beta_{-1}^s \leftarrow 1, \beta_0^s \leftarrow \alpha_{1, s},\bar{L}_s \leftarrow \frac{L}{B\alpha_{3, s}} + L_B$; \STATE $\eta_s \leftarrow \frac{1}{\bar{L}_s\alpha_{2,s}B}, \theta_s \leftarrow 1+\frac{\mu}{\bar{L}_sB^2\alpha_{2, s} + (B-1)\mu}$;
		\STATE $\dot{\nabla}^s \leftarrow \nabla f(\dot{\xB}^s)$; \label{eqn: full gradient}
		\STATE $\uB_0^{s} \leftarrow \xi - \gamma_{s} \hat{\zB}_0^{s} - \dot{\xB}^{s}$; \label{eqn: update between epochs II}
		\FOR{$j \leftarrow 1$ \TO $m$}
		\STATE $k \leftarrow sm+j$;
		\STATE sample mini batch $I$ of size $b$ and feature block $l$;
		\STATE $[\dot{\vB}_k]_l \leftarrow [\dot{\nabla}^s]_l + \frac{1}{b}\sum_{i\in I} \big([\nabla f_i(\clrg{\bar{\yB}_k})]_l - [\nabla f_i(\tilde{\xB}^s)]_l\big)$; \label{eqn_alg_gradient}
		\STATE $[\hat{\zB}_j^s]_l\!\!\leftarrow \!\prox_{\eta_s\PB_l}\!([\clrg{\bar{\zB}_k}]_l\! -\! \eta_s[\dot{\vB}_k]_l)\! -\! [\dot{\xB}^s]_l$, $[\hat{\zB}_j^s]_{\backslash l} \!\leftarrow\! [\hat{\zB}_{j-1}^s]_{\backslash l}$;
		\STATE $\uB_j^s \leftarrow \uB_{j-1}^s + \frac{\alpha_{2,s}B - \gamma_s}{\beta_{j-1}^s}(\hat{\zB}_j^s - \hat{\zB}_{j-1}^s)$; $\beta_j^s \leftarrow \alpha_{1,s}\beta_{j-1}^s$; \label{eqn: lazy update}
		\ENDFOR
		\STATE sample $\sigma$ from $\{1, \ldots, m\}$ with probability $\frac{\theta_s^{\sigma-1}}{\sum_{i=0}^{m-1}\theta_s^{i-1}}$; 
		\STATE $\dot{\xB}^{s+1} \leftarrow \clrg{\bar{\xB}_{sm+\sigma}}$;
		\STATE $\hat{\zB}_0^{s+1} \leftarrow \clrg{\bar{\zB}_k} - \dot{\xB}^{s+1}$, $\xi = \clrg{\bar{\xB}_k}$; \label{eqn: update between epochs}
		\ENDFOR
	\end{algorithmic}
\end{algorithm}
While ADSG has an accelerated convergence rate, naively implementing Algorithm \ref{alg: ADSG I} requires $\OM(d)$ computation in each inner loop due to the two \emph{coupling steps}, which compromises the low per-iteration complexity enjoyed by RBCD type methods.
Such quandary strikes all existing accelerated RBCD algorithm \clr{\cite{lin2015accelerated,nesterov2012efficiency,fercoq2015accelerated,lee2013efficient}}.
To bypass this dilemma, we cast Algorithm \ref{alg: ADSG I} in an equivalent but more practical form, ADSG II, with the inner loop complexity reduced to $\OM(\Omega)$.
ADSG II uses three auxiliary functions $\{\bar{\xB}_k, \bar{\yB}_k, \bar{\zB}_k\}$, marked with green in Algorithm \ref{alg: ADSG II} and defined here as
{\small
	\[\begin{aligned}
	\bar{\yB}_k =&~ \beta_{j-1}\uB_{j-1}^s + \gamma_s\hat{\zB}_{j-1}^s  + \dot{\xB}^s, &\mbox{Line \ref{eqn_alg_gradient}},\\
	\bar{\zB}_k =&~ \hat{\zB}_j^s + \dot{\xB}^s,&\mbox{Line \ref{eqn: update between epochs}},\\
	\bar{\xB}_k =&~ \beta_{j-1}\uB_j^s + \gamma_s\hat{\zB}_j^s  + \dot{\xB}^s, &\mbox{Line \ref{eqn: lazy update} and \ref{eqn: update between epochs}},
\end{aligned}\]}
with $k = sm+j$.
The following proposition shows the equivalence between Algorithm \ref{alg: ADSG I} and \ref{alg: ADSG II}.
\begin{proposition}
	If Algorithm \ref{alg: ADSG I} has the same input as Algorithm \ref{alg: ADSG II}, its iterates $\xB_k$, $\yB_k$, and $\zB_k$ equal to $\bar{\xB}_k$, $\bar{\yB}_k$, $\bar{\zB}_k$ respectively for all $k$.
\end{proposition}
\begin{proof}
	First, we prove that if at the beginning of the $s^{th}$ epoch, $\bar{\zB}_{sm} = \zB_{sm}$, $\bar{\xB}_{sm} = \xB_{sm}$, and $\dot{\xB}^s = \tilde{\xB}^s$ stand, then the proposition stand in the following iterations in that epoch.
	We prove with induction.
	Assume that the equivalence holds till the $(k-1)^{th}$ iteration.
	In the $k^{th}$ iteration, for $\bar{\yB}_k$ we have
		{\[	\begin{aligned}
			\yB_k &= \alpha_{1,s}\xB_{k-1} +\alpha_{2,s}{\zB}_{k-1} + \alpha_{3,s}\tilde{\xB}^s\\
			&= \alpha_{1,s}\bar{\xB}_{k-1} +\alpha_{2,s}\bar{\zB}_{k-1} + \alpha_{3,s}\dot{\xB}^s\\
			&= \alpha_{1,s}(\beta_{j-2}^s\uB_{j-1}^s + \gamma_s\hat{\zB}_{j-1}^s) + \alpha_{2,s}\hat{\zB}_{j-1}^s + \dot{\xB}^s \\
			&= \beta_{j-1}^s\uB_{j-1}^s + \gamma_s\hat{\zB}_{j-1}^s + \dot{\xB}^s = \bar{\yB}_k.
			\end{aligned}\]}
	since $\alpha_{1,s}\beta_{j-2}^s = \beta_{j-1}^s$ and $\alpha_{1,s}\gamma_s + \alpha_{2,s} = \gamma_s$.
	For $\bar{\zB}_k$, by induction we have $[\bar{\zB}_k]_{\backslash l} = [\bar{\zB}_{k-1}]_{\backslash l} = [\zB_{k -1}]_{\backslash l} = [\zB_{k}]_{\backslash l}$.
	Additionally, since $\dot{\nabla}^s = \tilde{\nabla}^s$ and $\yB_k = \bar{\yB}_k$, we have $[\vB_k]_l = [\dot{\vB}_k]_l$	and thus
	$[\bar{\zB}_k]_l = [\hat{\zB}_j^s + \dot{\xB}^s]_l =\!\prox_{\eta\PB_l}([\bar{\zB}_{k - 1}\! - \!\eta\dot{\vB}_k]_l) \! =\! \prox_{\eta\PB_l}([\zB_{k -1} \!-\! \eta\vB_k]_l)  = [\zB_k]_l$.
	For $\bar{\xB}_k$, we have
		{\[\begin{aligned}
			\xB_{k} &= \bar{\yB}_{k} + \alpha_{2,s}B(\hat{\zB}_{j}^s - \hat{\zB}_{j-1}^s) \\
			&=\beta_{j-1}^s\uB_{j-1}^s+\gamma_s\hat{\zB}_{j-1}^s + \alpha_{2,s}B(\hat{\zB}_{j}^s - \hat{\zB}_{j-1}^s) + \dot{\xB}^s\\
			&= \beta_{j-1}^s\uB_{j}^s + \gamma_s \hat{\zB}_{j}^s + \dot{\xB}^s = \bar{\xB}_{k}
			\end{aligned}\]} 
	by the updating rule of $\uB_k$ in line \ref{eqn: lazy update} in ADSG II.
	
	We then show that at the beginning of each epoch $\bar{\zB}_{sm} = \zB_{sm}$, $\bar{\xB}_{sm} = \xB_{sm}$, and $\dot{\xB}^s = \tilde{\xB}^s$ stand.
	For $\zB_0$, we have $\bar{\zB}_0 = \zB_0$ from the initialization.
	For $\zB_{sm}, s \geq 1$, we have $\zB_{sm} = \bar{\zB}_{sm}  = \hat{\zB}_0^{s} + \dot{\xB}^{s} = \bar{\zB}_{sm}$, where the first equation is from the induction in previous epoch, and the second equation is from the definition of $\hat{\zB}_0^{s}$ in line \ref{eqn: update between epochs} in ADSG II.
	For $\xB_0$, we clearly have $\xB_0 = \bar{\xB}_0$ by the initialization.
	For $\xB_{sm}, s \geq 1$, we have {\small $\xB_{sm} = \xi = \beta_{-1}^{s}\uB_0^{s} + \gamma_{s}\hat{\zB}_0^{s} + \dot{\xB}^{s} = \bar{\xB}_{sm},$}
	where the first equation is from the induction in previous epoch, and the second equation is from line \ref{eqn: update between epochs II} in ADSG II.
	$\dot{\xB}^s = \tilde{\xB}^s$ because $\xB_k = \bar{\xB}_k$ for all $k$ in that epoch.	
	Thus we have the result.
\end{proof}

\subsection{Avoiding Numerical Issue}
\begin{algorithm}[t]
	\caption{ADSG III}
	\begin{algorithmic}[1]
		\label{alg: ADSG III}
		\REQUIRE $m, \xB_0, \alpha_{1,0}, \alpha_{2, 0}$
		\STATE $\uB_0^0 = \hat{\zB}_0^0 \leftarrow 0, \dot{\xB}^0 \leftarrow \xB_0, k = 0$;
		\FOR{$s \leftarrow 0$ \TO $S$}
		\STATE $\bar{L}_s = \frac{L_Q}{B\alpha_{3, s}} + L_B, \eta_s = \frac{1}{\bar{L}_s\alpha_{2,s}B}$;
		\STATE $\dot{\mu}^s = \nabla f(\dot{\xB}^s)$;
		\STATE $\gamma_s = \frac{\alpha_{2,s}}{\alpha_{2,s} + \alpha_{3,s}}$;
		\STATE $\hat{\zB}_0^s = \bar{\zB}_k - \dot{\xB}^s$;		
		\STATE $\xi_0^s = \bar{\xB}_k - \gamma_s \hat{\zB}_0^s - \dot{\xB}^s, \omega = 1^B$; 
		\FOR{$j \leftarrow 1$ \TO $m$}
		\STATE $k = (sm)+j$;

		\STATE sample $i$ from $\{1, \ldots, n\}$ and $l$ from $\{1, \ldots, B\}$;
		\STATE $\tilde{\nabla}_k = \dot{\mu}^s + \nabla f_i(\bar{\yB}_k) - \nabla f_i(\dot{\xB}^s)$; \label{eqn: partial gradient}
		\STATE $[\hat{\zB}_j^s]_l = \prox_{\eta\PB_l}([\bar{\zB}_k - \eta\tilde{\nabla}_k]_l) - [\dot{\xB}^s]_l, [\hat{\zB}_j^s]_{\backslash l} = [\hat{\zB}_{j-1}^s]_{\backslash l}$;
		\STATE $[\xi_j^s]_l = \alpha_{1,s}^{\omega_l}[\xi_{j-1}^s]_l + (\alpha_{2,s}B - \gamma_s)[\hat{\zB}_j^s - \hat{\zB}_{j-1}^s]_l, [\xi_j^s]_{\backslash l} = [\xi_{j-1}^s]_{\backslash l}$;
		\STATE $\omega_l = 0$, $\omega_i = \omega_i + 1, i \neq l$;
		\ENDFOR
		\STATE Sample $\bar{\sigma}_s$ from $\{1, \ldots, m\}$ uniformly;
		\STATE $\dot{\xB}^{s+1} = \xi^s_{\bar{\sigma}_s}+\gamma_s\hat{\zB}_{\bar{\sigma}_s}^s + \dot{\xB}^s$;
		\ENDFOR
	\end{algorithmic}
\end{algorithm}
Since $\beta_j^s$ decreases exponentially (line \ref{eqn: lazy update} in Algorithm \ref{alg: ADSG II}), the computation of $\uB_j^s$ involving the inversion of $\beta_j^s$ can be numerically unstable.
To overcome this issue, we can simply keep their product $\beta^s_{j-1}\uB_j^s = \Xi_j^s\in \RBB^d$ rather than themselves separately to make the computation numerically tractable.
Consequently, the functions $\bar{\yB}_k$ and $\bar{\xB}_k$ are transformed into
\begin{align}
	\bar{\yB}_k =&~ \alpha_{1, s}\Xi_{j-1}^s + \gamma_s\hat{\zB}_{j-1}^s  + \dot{\xB}^s,\\
	\bar{\xB}_k =&~ \Xi_j^s + \gamma_s\hat{\zB}_j^s  + \dot{\xB}^s.
\end{align}
Since exactly computing $\Xi_j^s$ involves full vector operations, we maintain two vectors $\xi_j^s \in \RBB^d$ and $\omega_j^s \in \RBB^B$ instead so that the following lazy update strategy can be utilized.

At the beginning of each epoch, we initialize a count vector $\omega_0^s \in \RBB^B$ to be a zero vector and set $\xi_0^s = \uB_0^s$.
In the $j^{th}$ iteration, suppose $l$ is the block being selected. We do the follow steps
\begin{enumerate}
	\item $[\xi_j^s]_l = \alpha_{1,s}^{[\omega_{j-1}^s]_l+1}[\xi_{j-1}^s]_l + (\alpha_{2,s}B - \gamma_s)[\hat{\zB}_j^s - \hat{\zB}_{j-1}^s]_l$, $[\xi_j^s]_{\backslash l} = [\xi_{j-1}^s]_{\backslash l}$, 
	\item $[\omega_j^s]_l = 0$, $[\omega_j^s]_{\backslash l} = [\omega_{j-1}^s]_{\backslash l} + 1$.
\end{enumerate}
\begin{proposition}
	Maintaining $\omega_j^s$ and $\xi_j^s$ as above, then we have $[\Xi_j^s]_l = \alpha_{1,s}^{[\omega_j^s]_l}[\xi_j^s]_l, \forall l \in [B]$.
	\label{proposition_numerical_issue}
\end{proposition}
\begin{proof}
	We prove via induction.
	By setting $\beta_{-1}^s = 1$ for all $s$, we have $\Xi_0^s = \uB_0^s = \xi_0^s$. 
	Assume the conclusion holds up to the $j^{th}$ iteration. In the ${(j+1)}^{th}$ iteration, let $l$ be the block being sampled. For any $i\neq l$, $[\Xi_{j+1}^s]_i = \alpha_{1,s}[\Xi_j^s]_i = \alpha_{1,s}^{[\omega_j^s]_i+1}[\xi_j^s]_i = \alpha_{1,s}^{[\omega_{j+1}^s]_i}[\xi_{j+1}^s]_i$. Additionally, 
	\begin{align*}
		[\Xi_{j+1}^s]_l =& \alpha_{1,s}[\Xi_j^s]_l + (\alpha_{2,s}B - \gamma_s)[\hat{\zB}_{j+1}^s - \hat{\zB}_j^s]_l \\
		=& \alpha_{1,s}^{[\omega_j^s]_i+1}[\xi_j^s]_l + (\alpha_{2,s}B - \gamma_s)[\hat{\zB}_{j+1}^s - \hat{\zB}_j^s]_l = [\xi_{j+1}^s]_l,
	\end{align*}
	due to the definition of $[\xi_{j+1}^s]_l$ and that $[\omega_{j+1}^s]_l = 0$
\end{proof}
While line \ref{eqn: partial gradient} of Algorithm \ref{alg: ADSG III} computes $[\nabla f_i(\bar{\yB}_k)]_l$, the exact computation of $\Xi_{j-1}^s$ can be avoided in the ERM setting. We have $[\nabla f_i(\bar{\yB}_k)]_l = \nabla \phi_i(\aB_i^\top\bar{\yB}_k)[\aB_i]_l$ and therefore three inner products are involved: $\aB_i^\top\dot{\xB}^s$, $\aB_i^\top\hat{\zB}_j^s$, and $\aB_i^\top \Xi_{j-1}^s$, by recalling that $\aB_i^\top\bar{\yB}_k = \aB_i^\top(\alpha_{1,s}\Xi_{j-1}^s + \gamma_s\hat{\zB}_{j-1}^s  + \dot{\xB}^s)$. We can simply compute $\aB_i^\top \Xi_{j-1}^s$ by $\aB_i^\top \Xi_{j-1}^s = \sum_{l=1}^B \alpha_{1,s}^{[\omega_{j-1}^s]_l} [\aB_i]_l^\top [\xi_{j-1}^s]_l$, which is $\OM(\rho d + B)$.

By using all the lazy update strategies we discussed above, the exact computation of $\Xi_j^s$ only happens at the end of each epoch and we are able to avoid the full vector operation without introducing numerical issue.

\subsection{Overall Computational Complexity}
%
%
We discuss the detailed implementation of ADSG III when solving ERM problems.
For simplicity, the mini batch size $b$ is set to $1$.
In line \ref{eqn_alg_gradient}, $[\nabla f_i(\bar{\yB}_k)]_l = \nabla \phi_i(\aB_i^\top\bar{\yB}_k) [\aB_i]_l$, where we compute each term in $\aB_i^\top\bar{\yB}_k = \alpha_{1, s}\aB_i^\top\Xi_{j-1}^s + \gamma_s\aB_i^\top\hat{\zB}_{j-1}^s  + \aB_i^\top\dot{\xB}^s$ separately.
\begin{enumerate}
	\item In the first term, since we record $[\Xi_j^s]_l = \alpha_{1,s}^{[\omega_j^s]_l}[\xi_j^s]_l, \forall l \in [B]$ according to Proposition \ref{proposition_numerical_issue}, we compute $\psi_l = [\aB_i]_l^\top[\xi_{j-1}^s]_l$ for every $l \in [B]$ and then computes $\aB_i^\top\Xi_{j-1}^s = \sum_{l=1}^B \alpha_{1,s}^{[\omega_{j-1}^s]_l}\psi_l$.
	Therefore, we have $\OM(\rho d + B)$ computation for the first term.
	\item The second term can be done in $\OM(\rho d)$.
	\item We can save the third term when computing the gradient at the snapshot point $\tilde{\xB}^s$, and hence the third term takes $\OM(1)$.
\end{enumerate}
All in all, line \ref{eqn_alg_gradient} takes $\OM(\rho d + B)$.
Line \ref{eqn: lazy update} takes $\OM(\Omega)$ because only the $l^{th}$ block is updated.
Consequently, we have $\OM(\rho d + \Omega + B)$ from every inner loop.
In general $\rho$ is small in practical problems (Table \ref{table: statistics} gives $\rho$, the sparsity, of the used datasets), $\Omega$ dominates the rest two terms as long as $B \leq \sqrt{d}$ and $B \leq 1/\rho$.
For a moderate $B$, the per-epoch complexity of Algorithm \ref{alg: ADSG II} is $\OM(Bn\Omega) = \OM(dn)$.

Combining the above per-epoch complexity analysis and the convergence rate, the overall computational complexity of ADSG is $\OM(d(n + \sqrt{n/\kappa})\log1/\epsilon)$ in strongly convex case, and $\OM(d(n+\sqrt{nLD/\epsilon})\log1/\epsilon)$ in general convex case.

\section{Experiments}
	\begin{table}[b]
	\centering
	\caption{Statistics of datasets.}
	\small
	\begin{tabular}{|c|c|c|c|}
		\hline
		Dataset    &      n       &      d      &  sparsity  \\ \hline
		news20-binary &   $19,996$   & $1,355,191$ & $0.0336\%$ \\ \hline
		kdd2010-raw  & $19,264,097$ & $1,129,522$ &    $0.0008\%$    \\ \hline
		avazu-app   & $14,596,137$ & $999,990$ &    $0.0015\%$    \\ \hline
		url-combined  & $2,396,130$  & $3,231,961$ & $0.0036\%$ \\ \hline
	\end{tabular}
	\label{table: statistics}
	\end{table}
	\begin{figure}[!t]
		\centering
		\begin{tabular}{c|c|c|c}
			\includegraphics[width = .22\columnwidth]{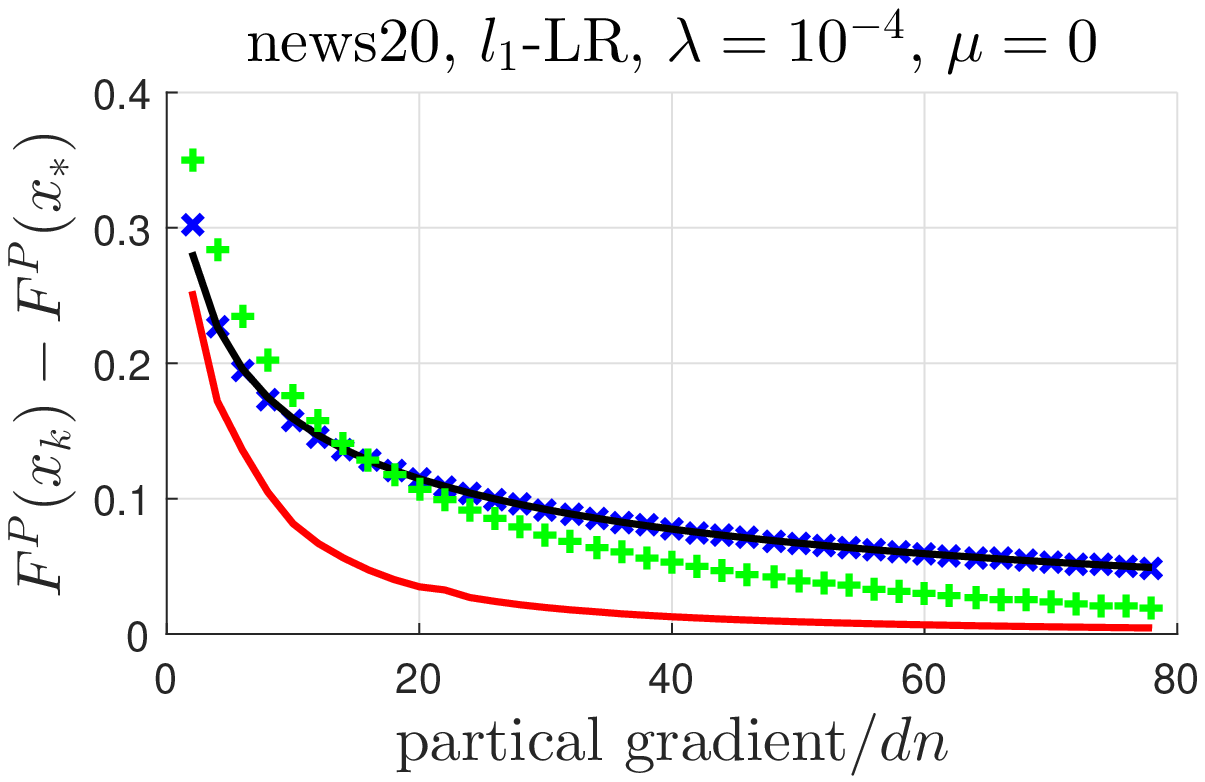}& 
			\includegraphics[width = .22\columnwidth]{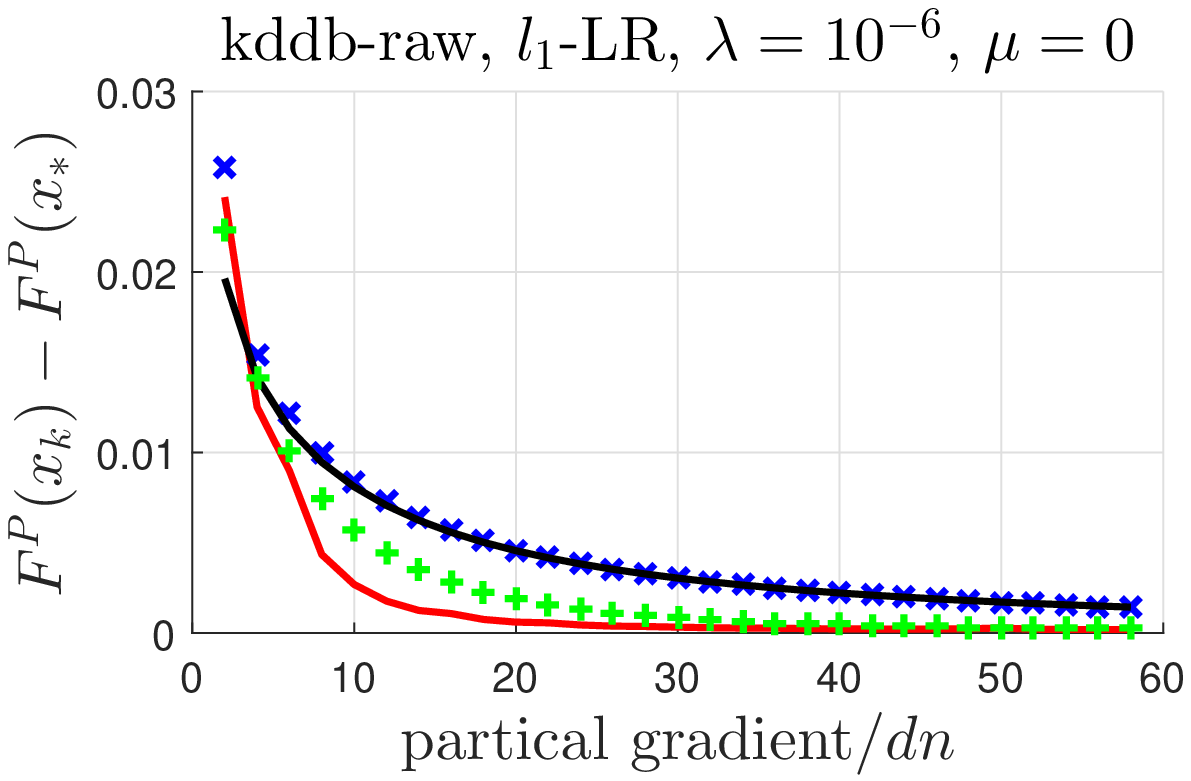}& 
			\includegraphics[width = .22\columnwidth]{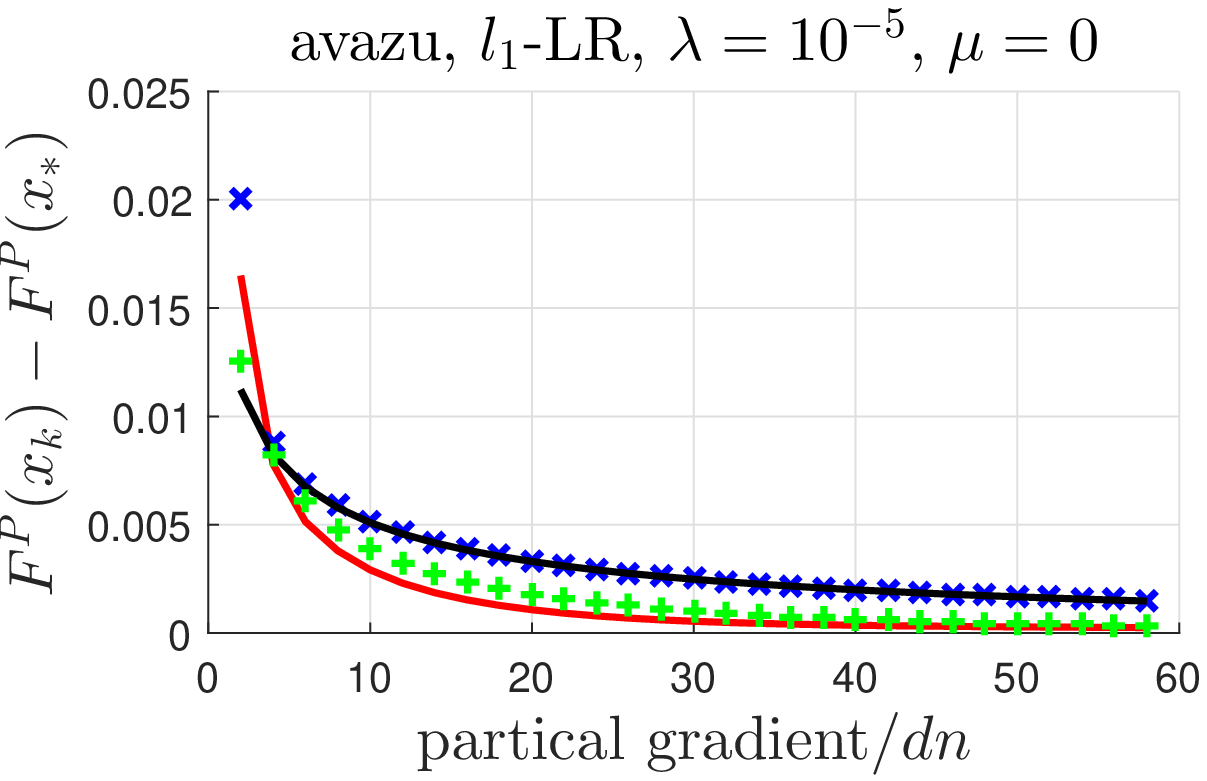}&
			\includegraphics[width = .22\columnwidth]{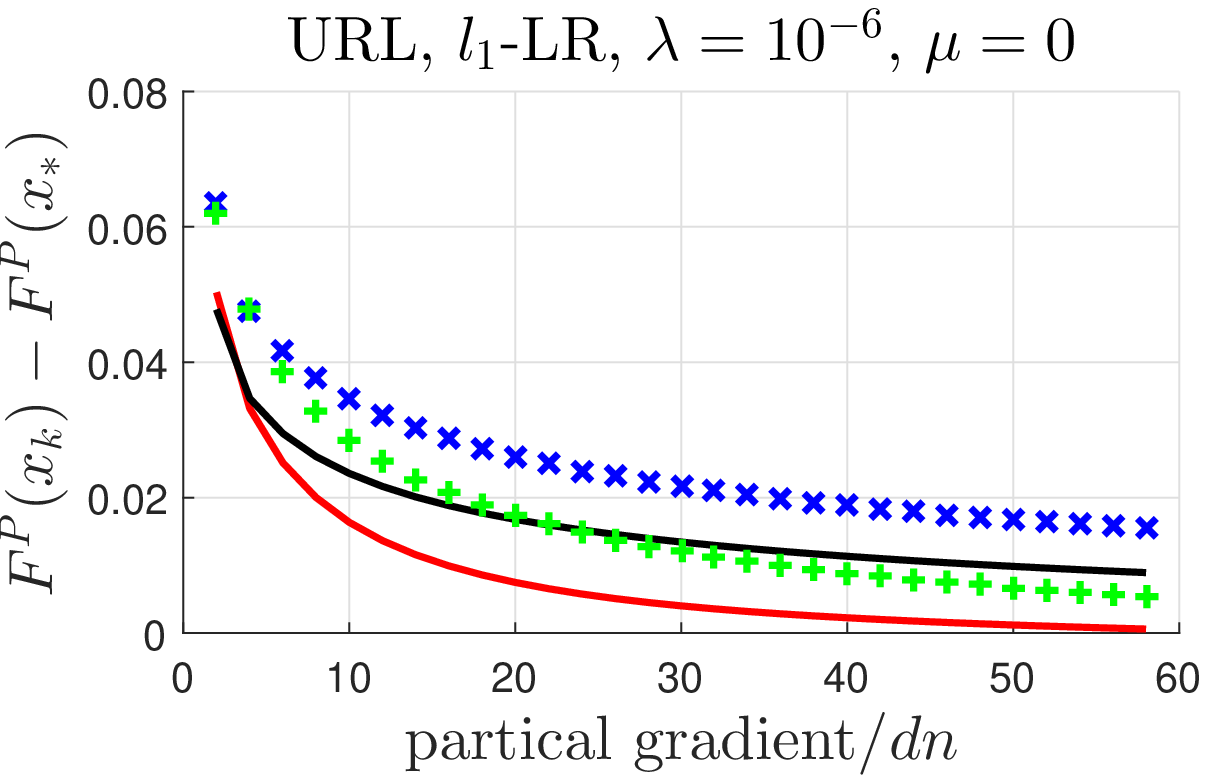}
			\\
			\includegraphics[width = .22\columnwidth]{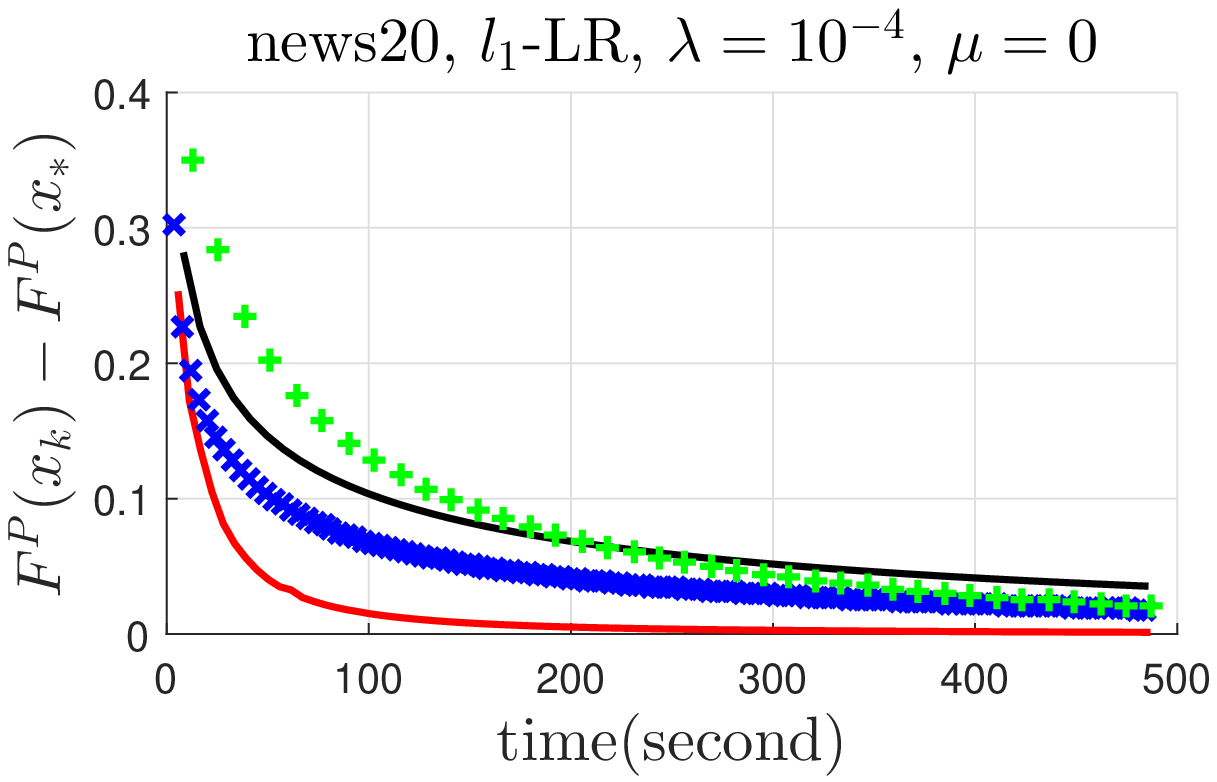}& 
			\includegraphics[width = .22\columnwidth]{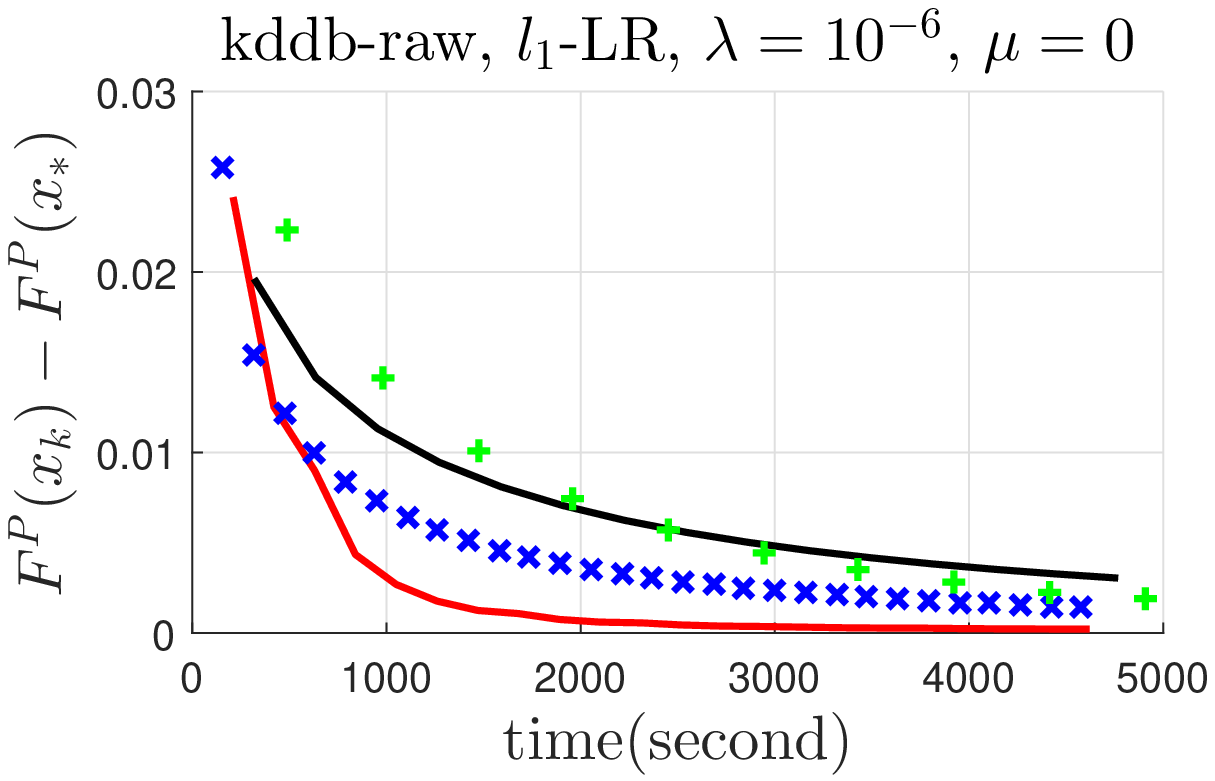}& 
			\includegraphics[width = .22\columnwidth]{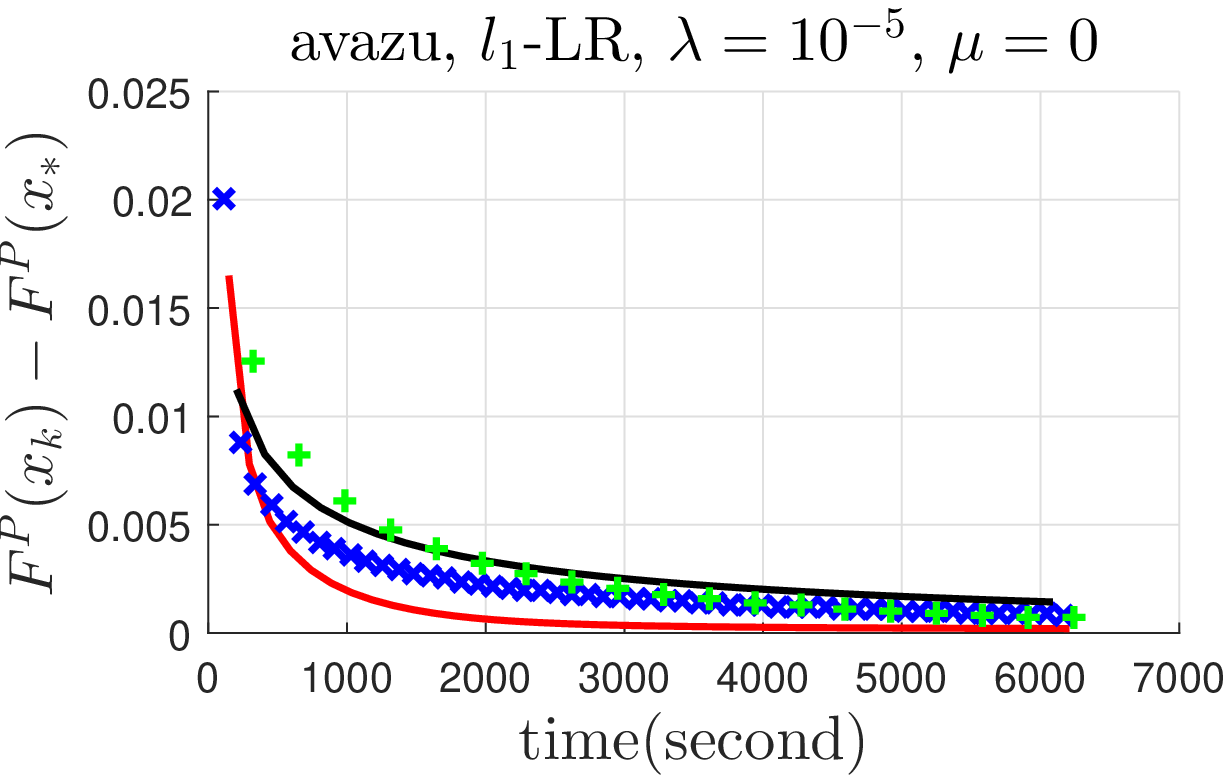}&
			\includegraphics[width = .22\columnwidth]{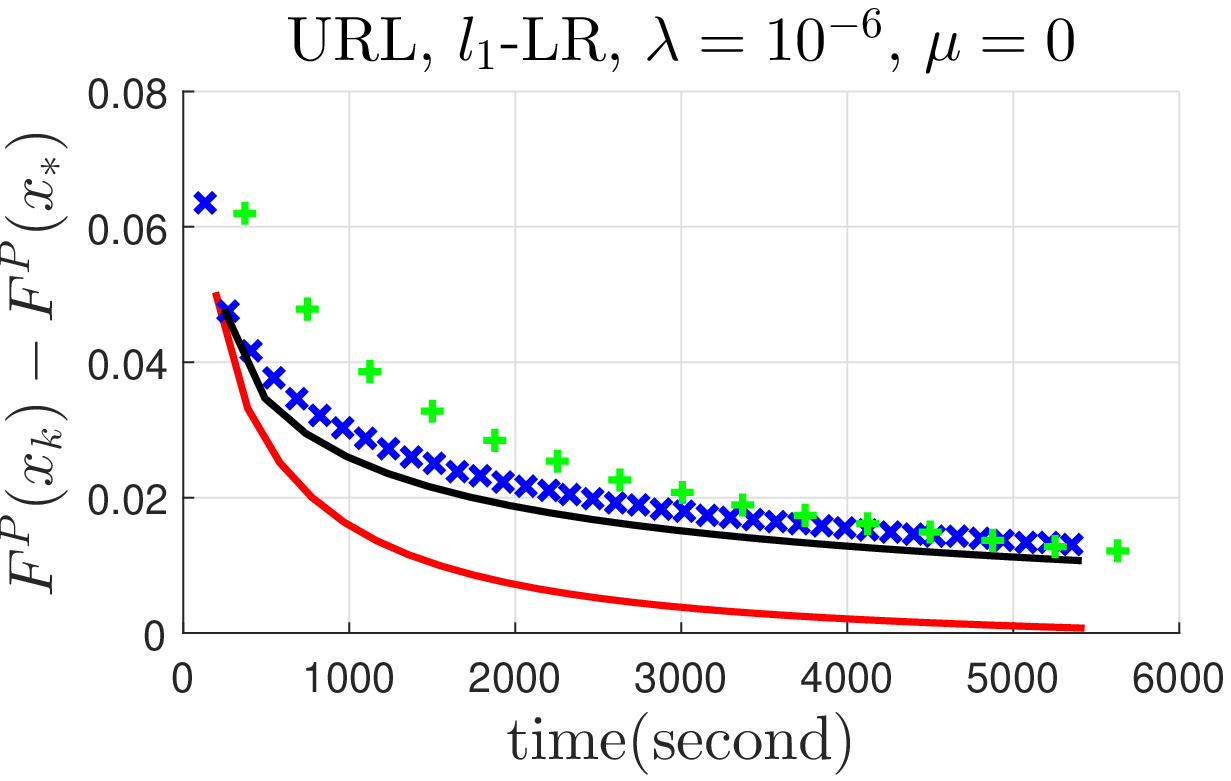}
			\\
			\includegraphics[width = .22\columnwidth]{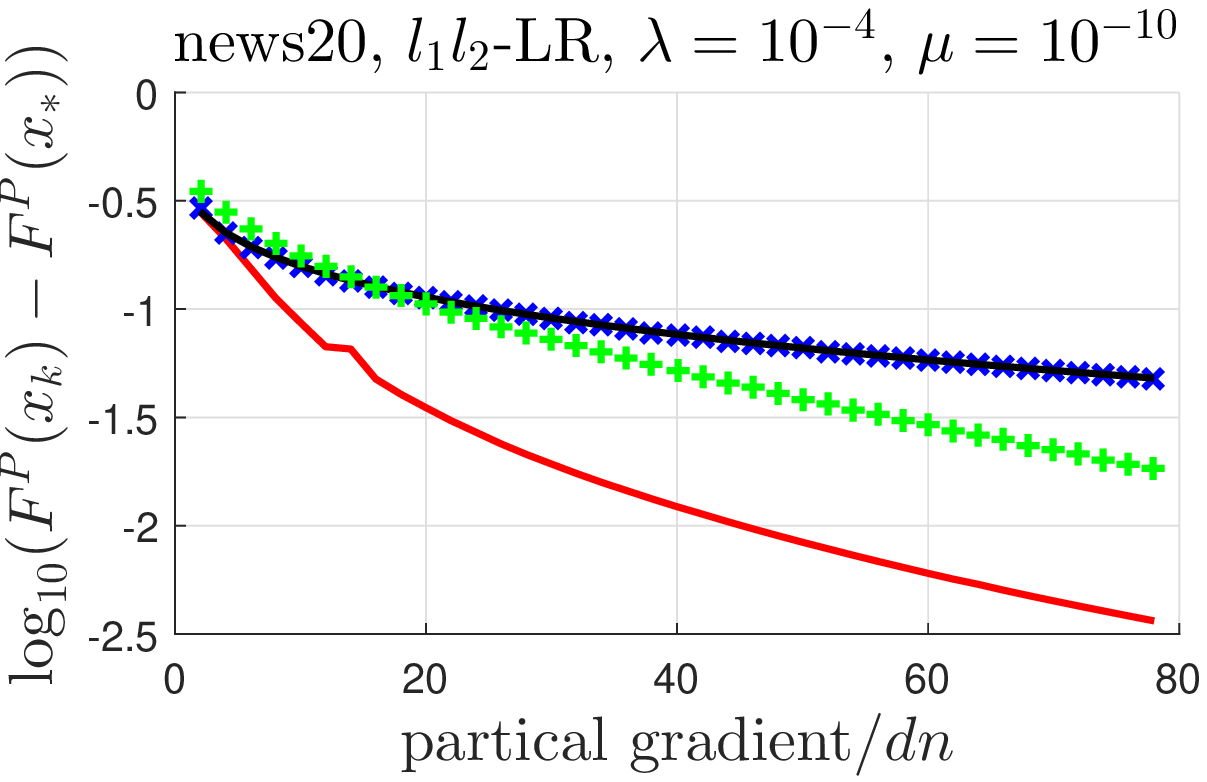}& 
			\includegraphics[width = .22\columnwidth]{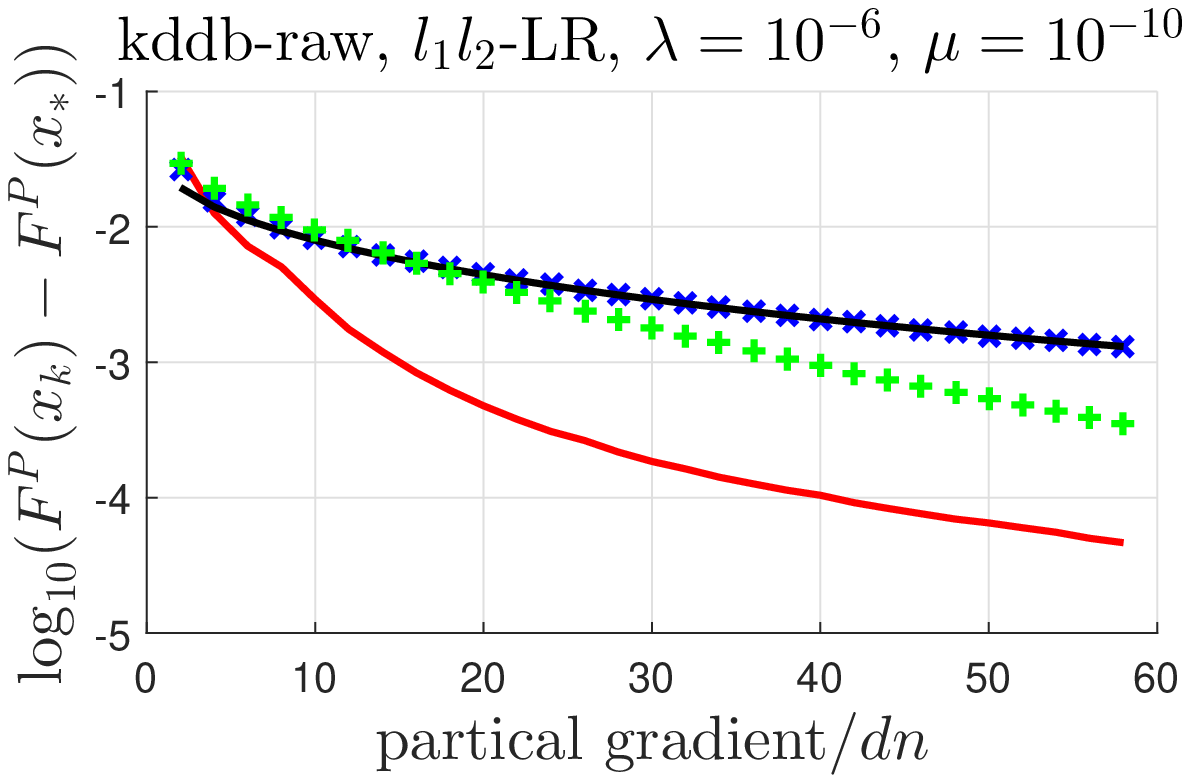}& 
			\includegraphics[width = .22\columnwidth]{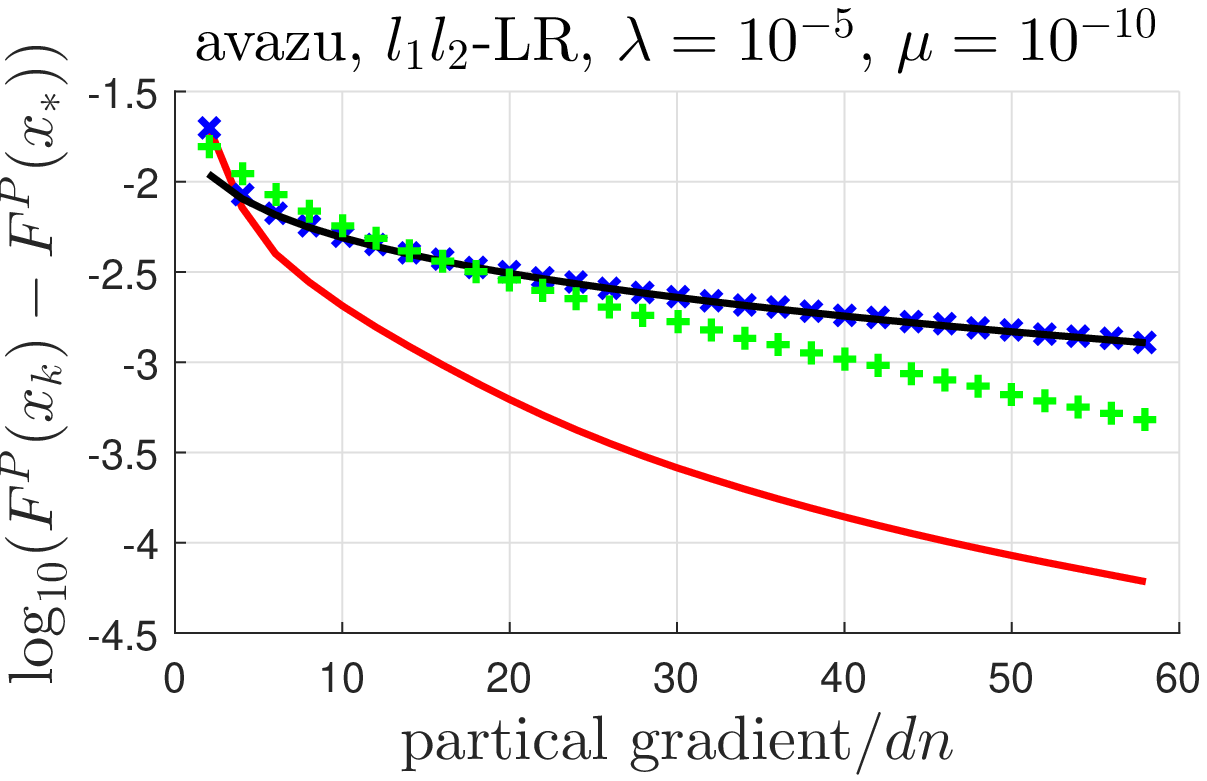}&
			\includegraphics[width = .22\columnwidth]{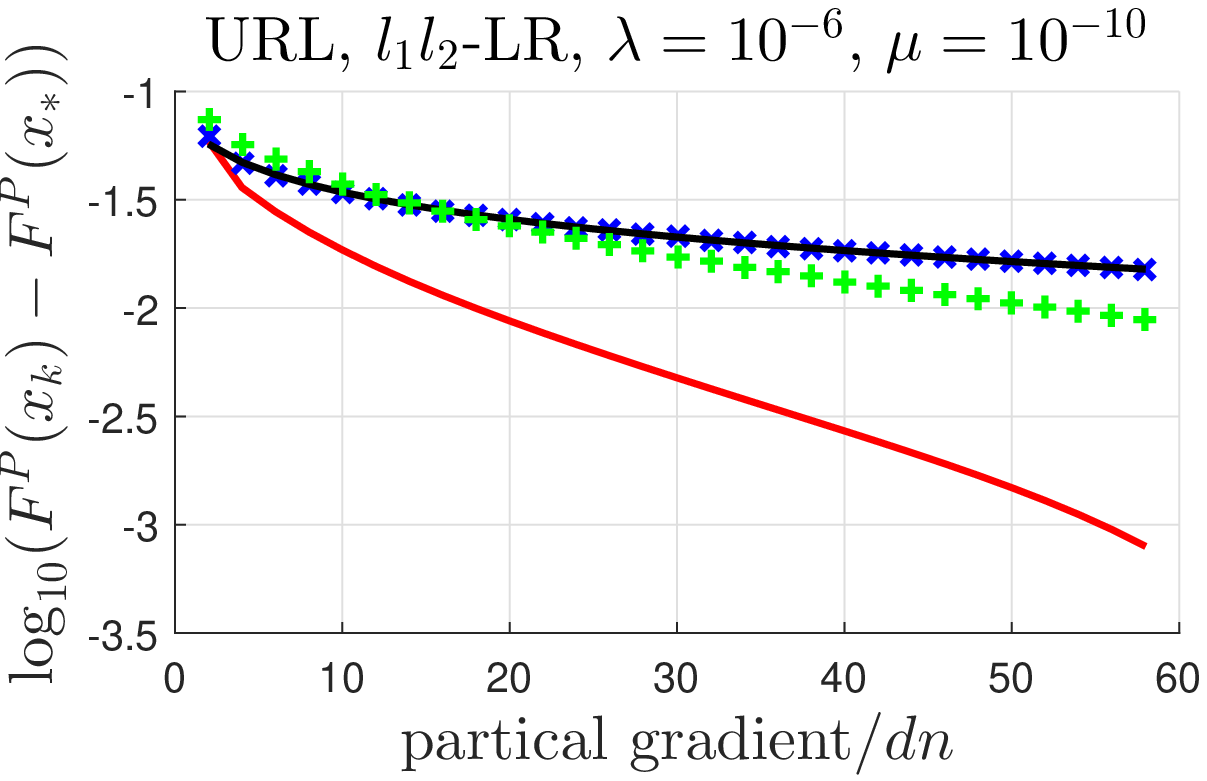}
			\\
			\includegraphics[width = .22\columnwidth]{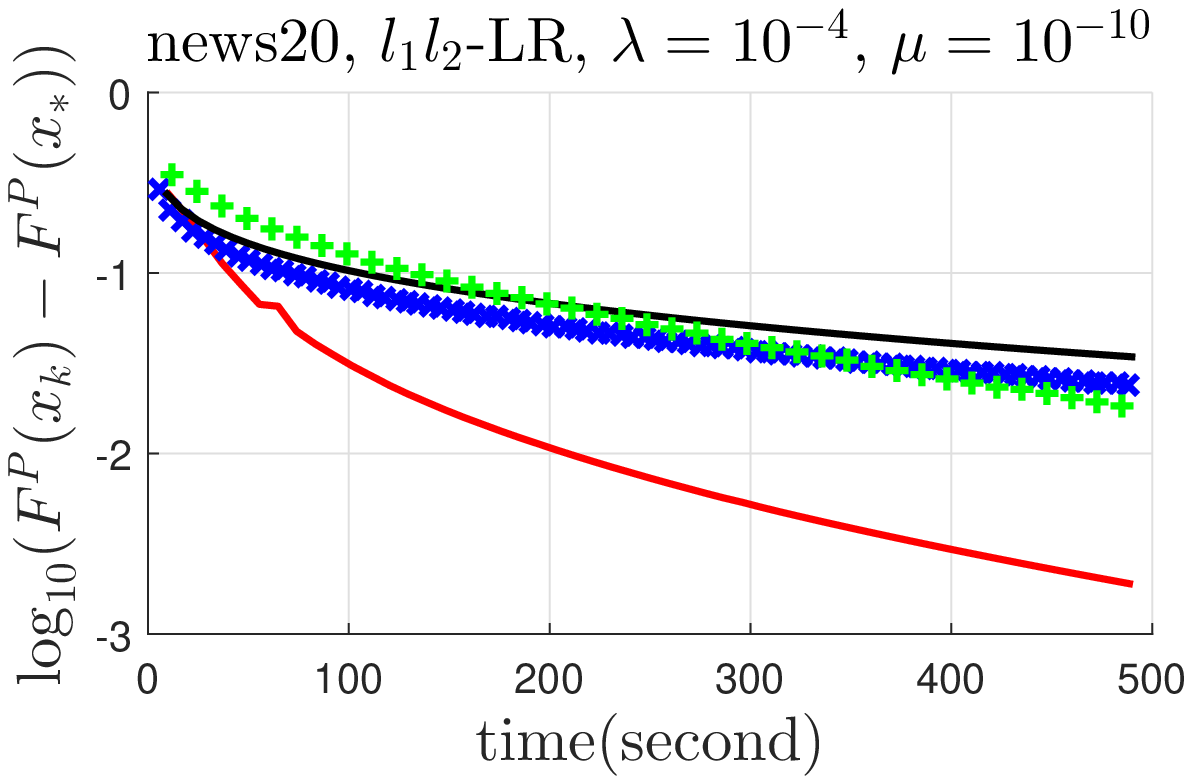}& 
			\includegraphics[width = .22\columnwidth]{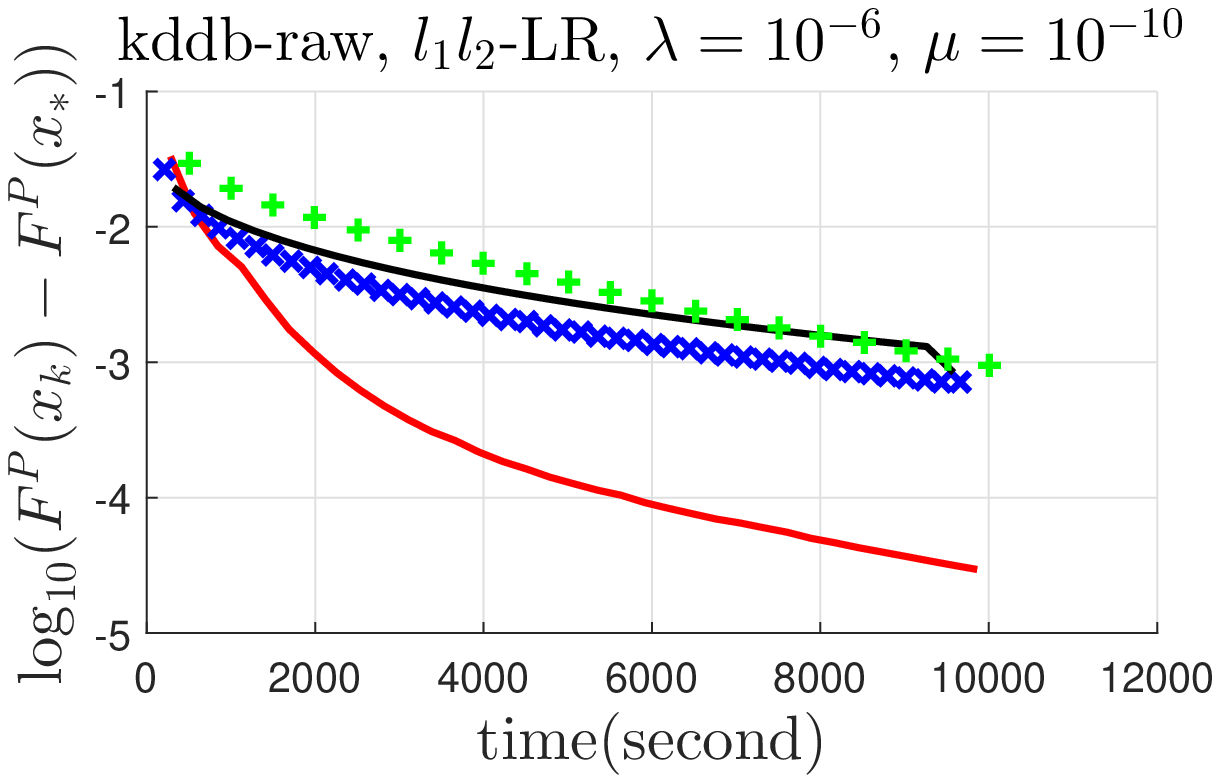}& 
			\includegraphics[width = .22\columnwidth]{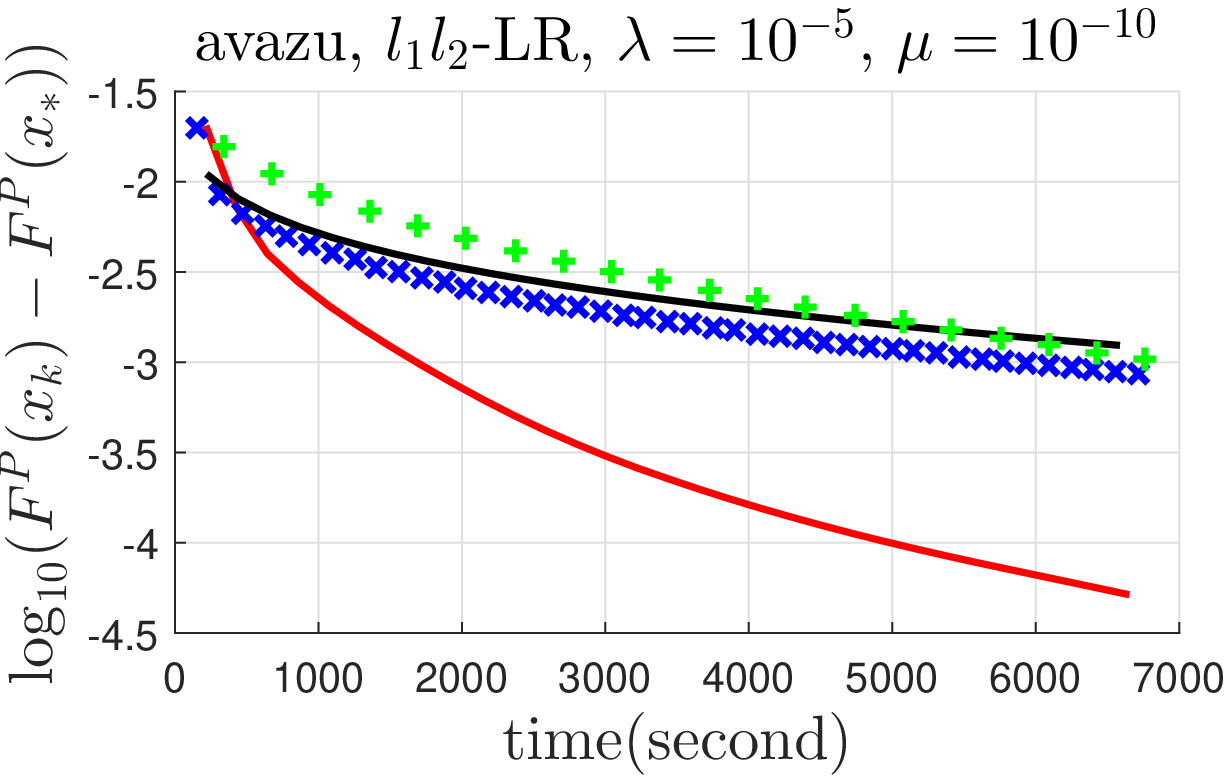}&
			\includegraphics[width = .22\columnwidth]{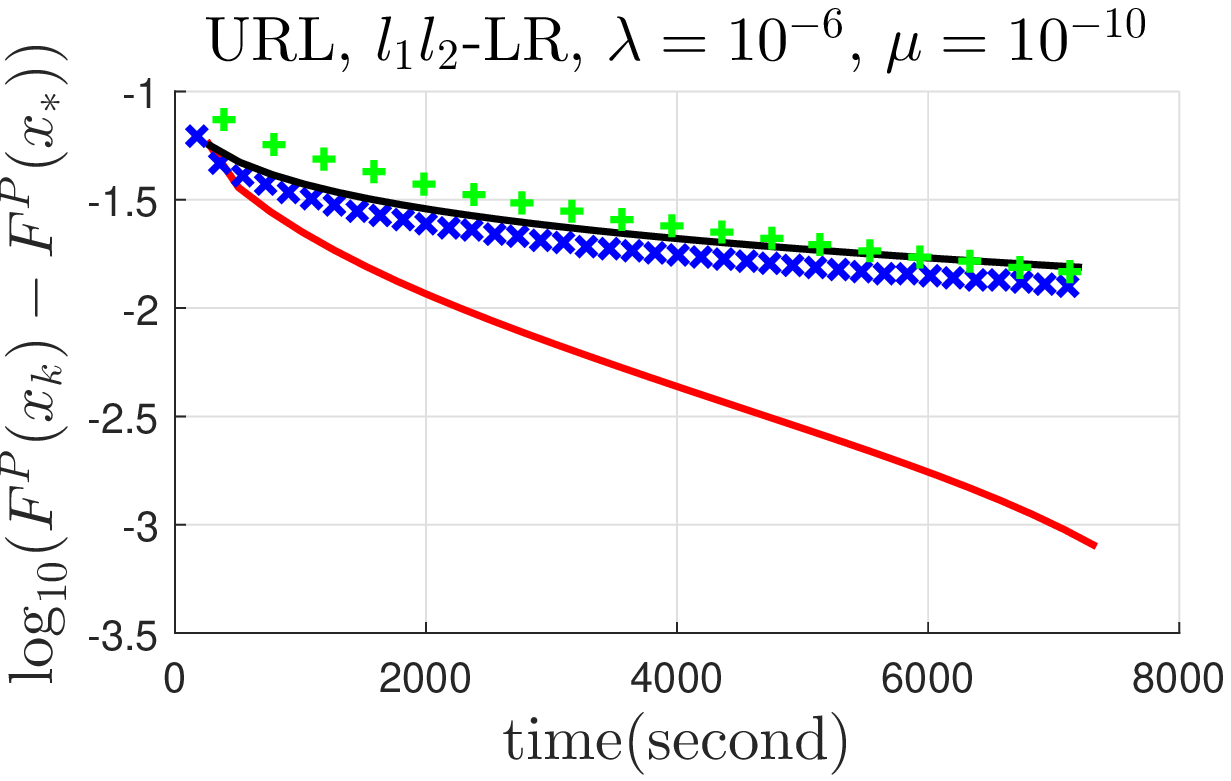}
			\\
			\includegraphics[width = .22\columnwidth]{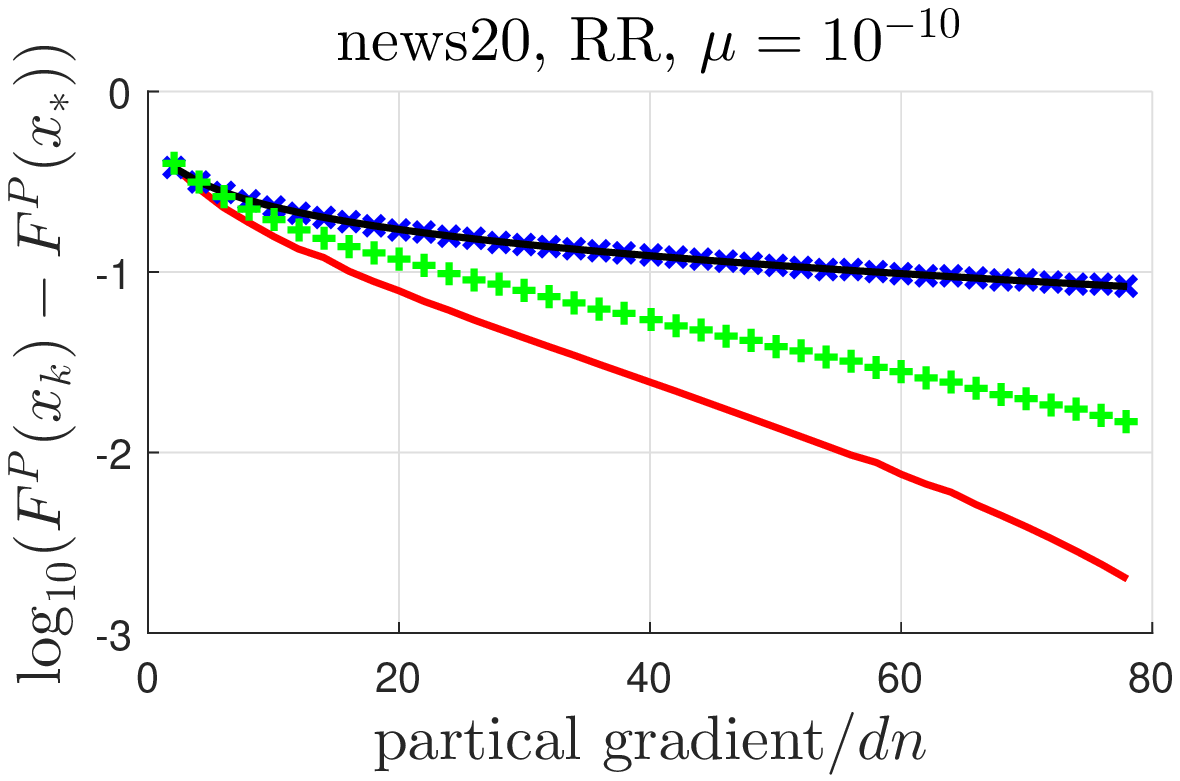}& 
			\includegraphics[width = .22\columnwidth]{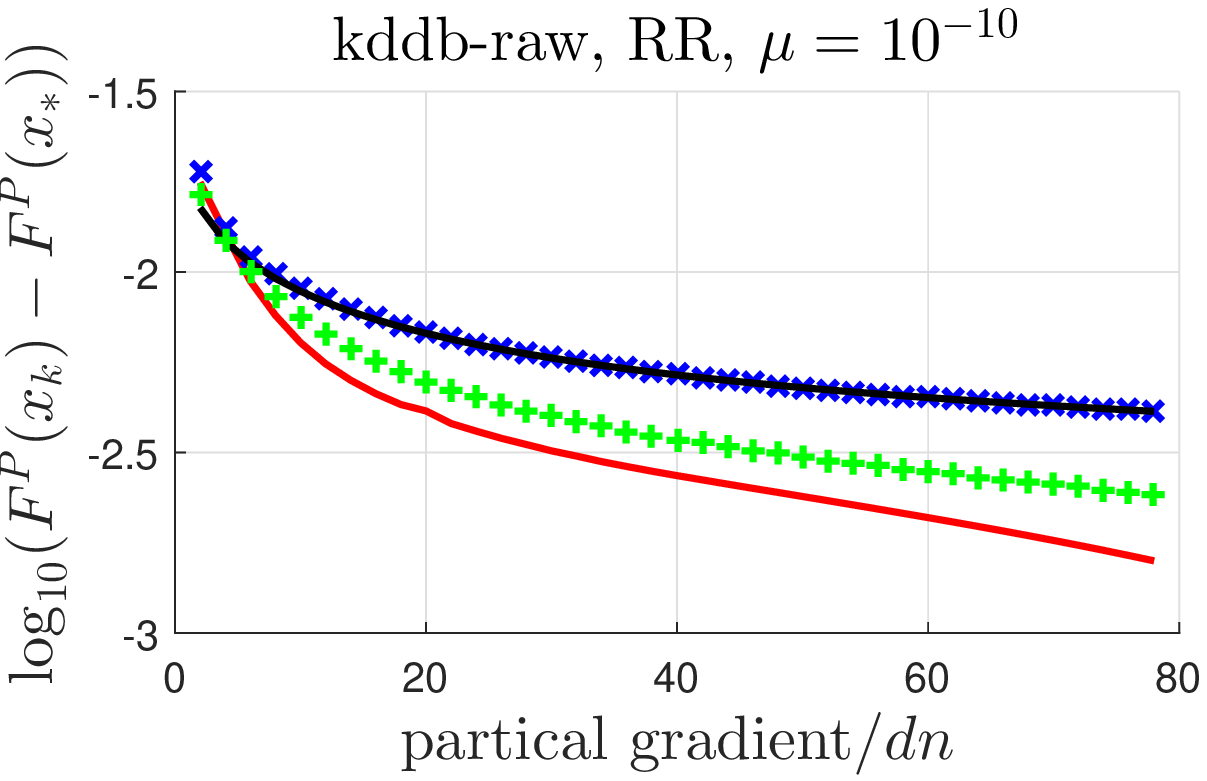}& 
			\includegraphics[width = .22\columnwidth]{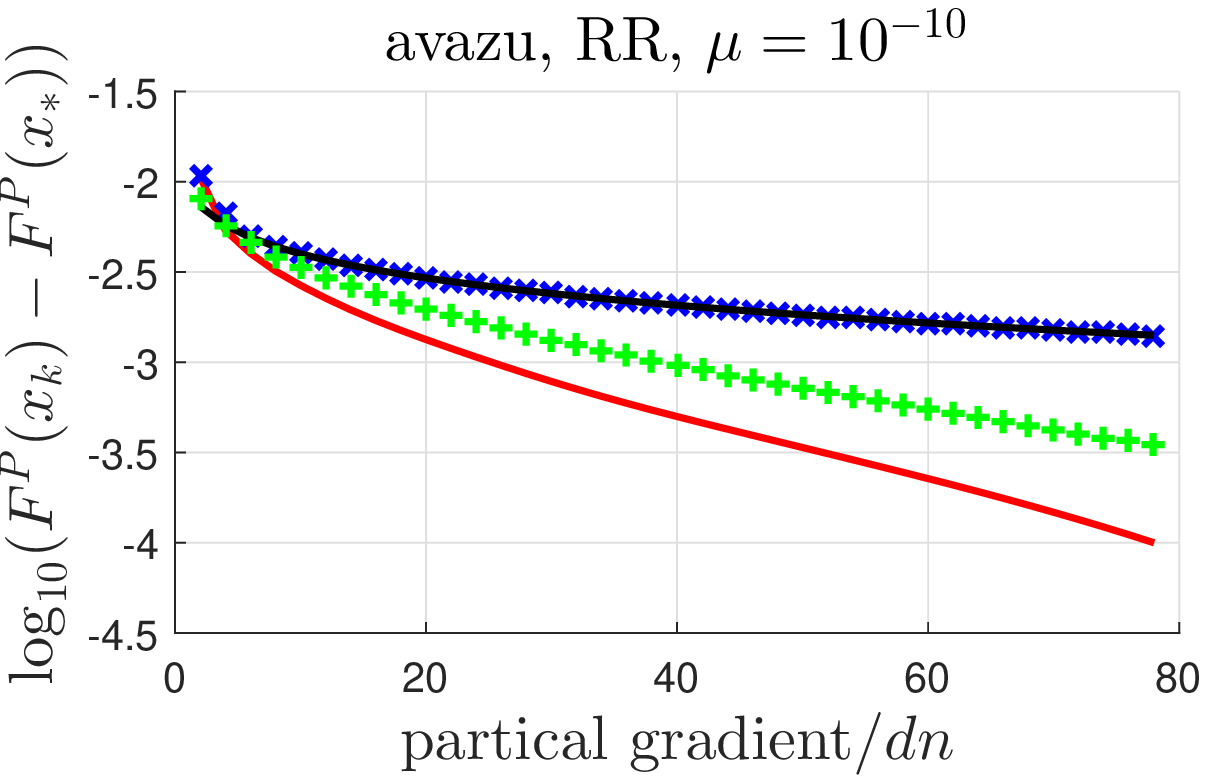}&
			\includegraphics[width = .22\columnwidth]{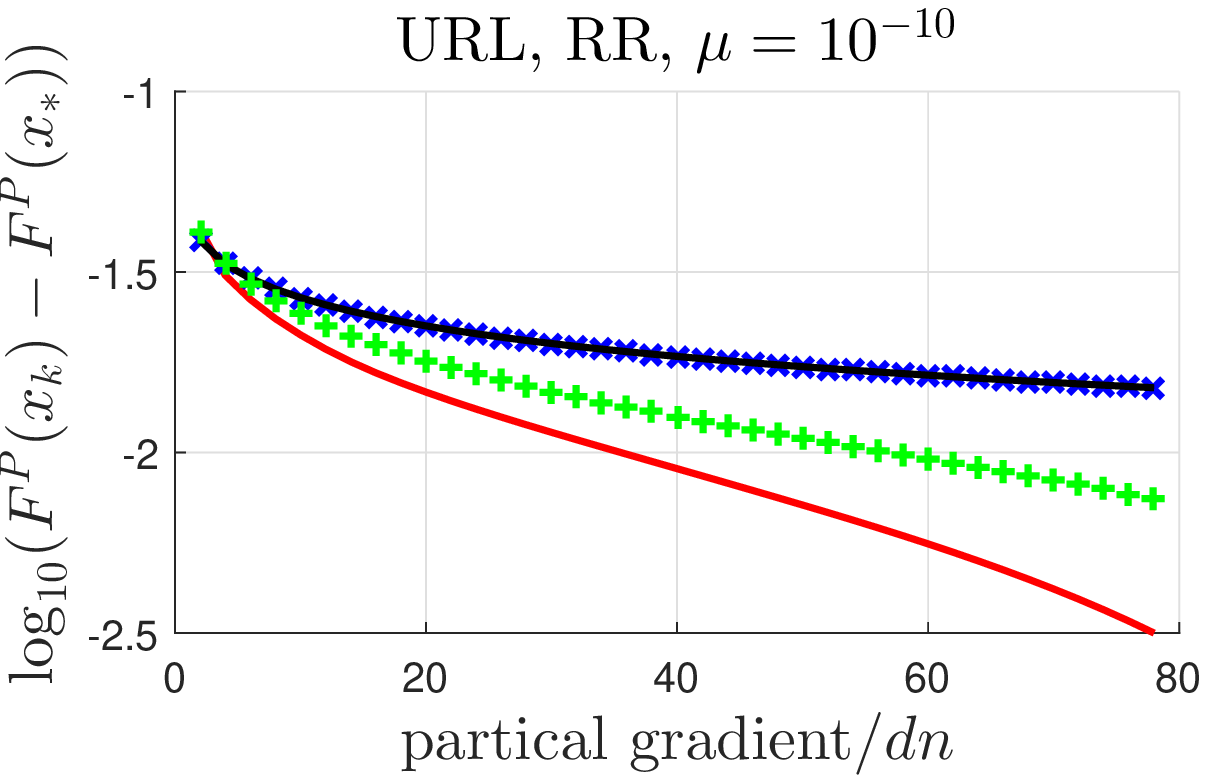}
			\\
			\includegraphics[width = .22\columnwidth]{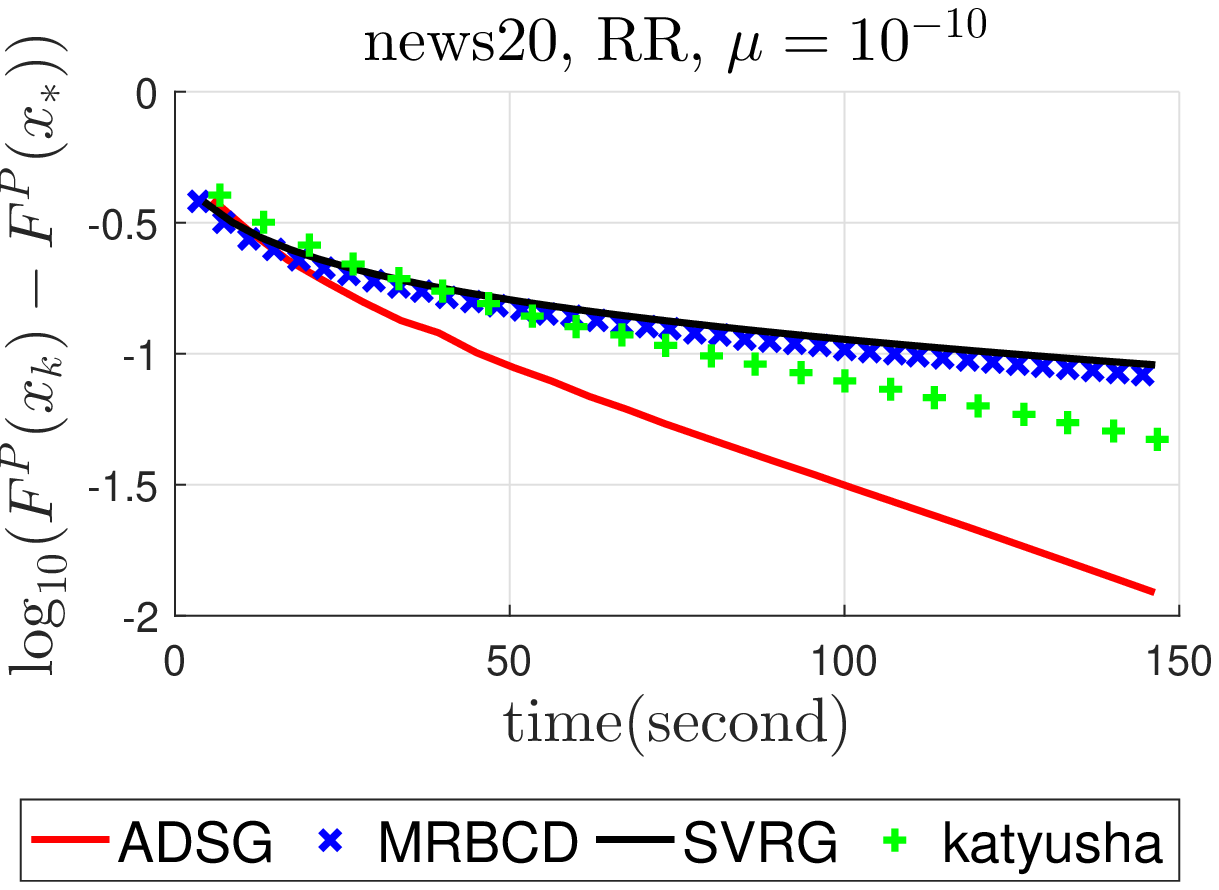}& 
			\includegraphics[width = .22\columnwidth]{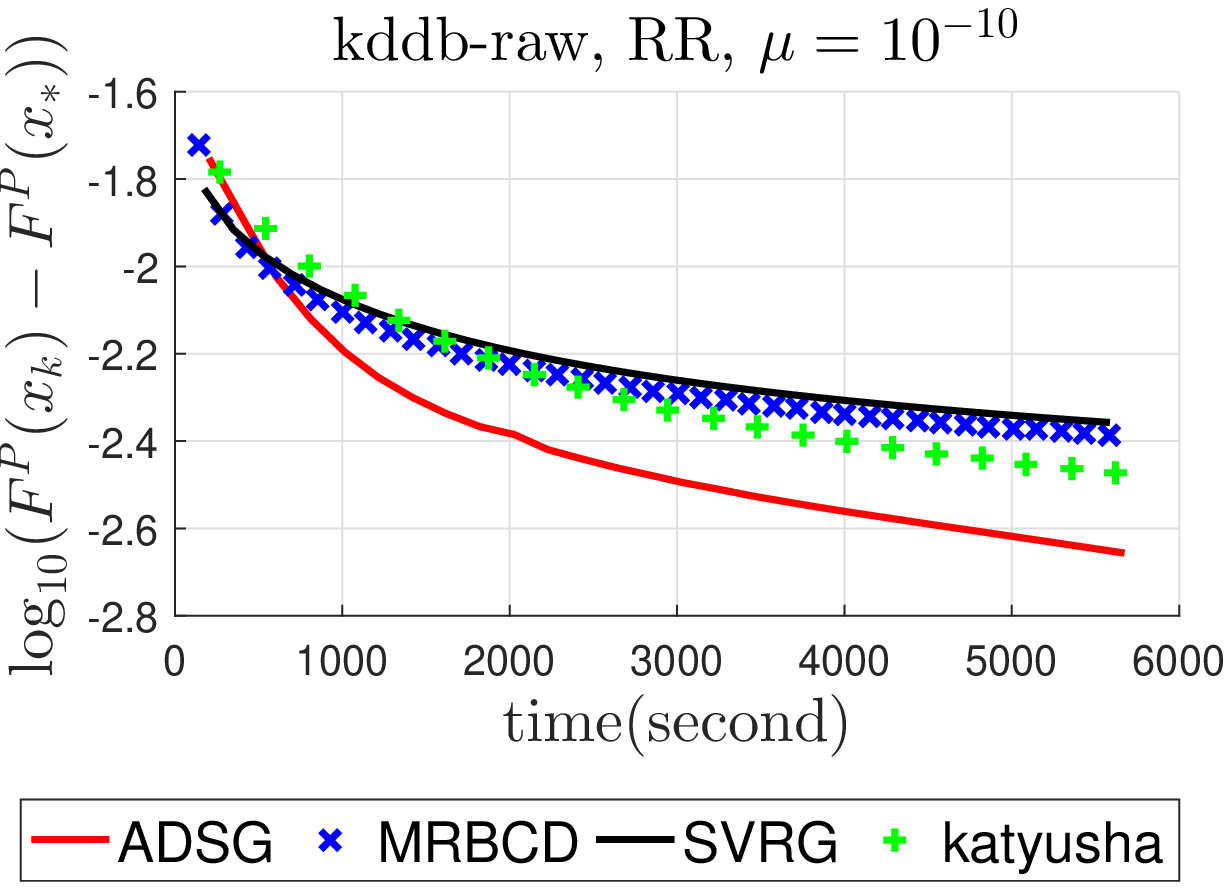}& 
			\includegraphics[width = .22\columnwidth]{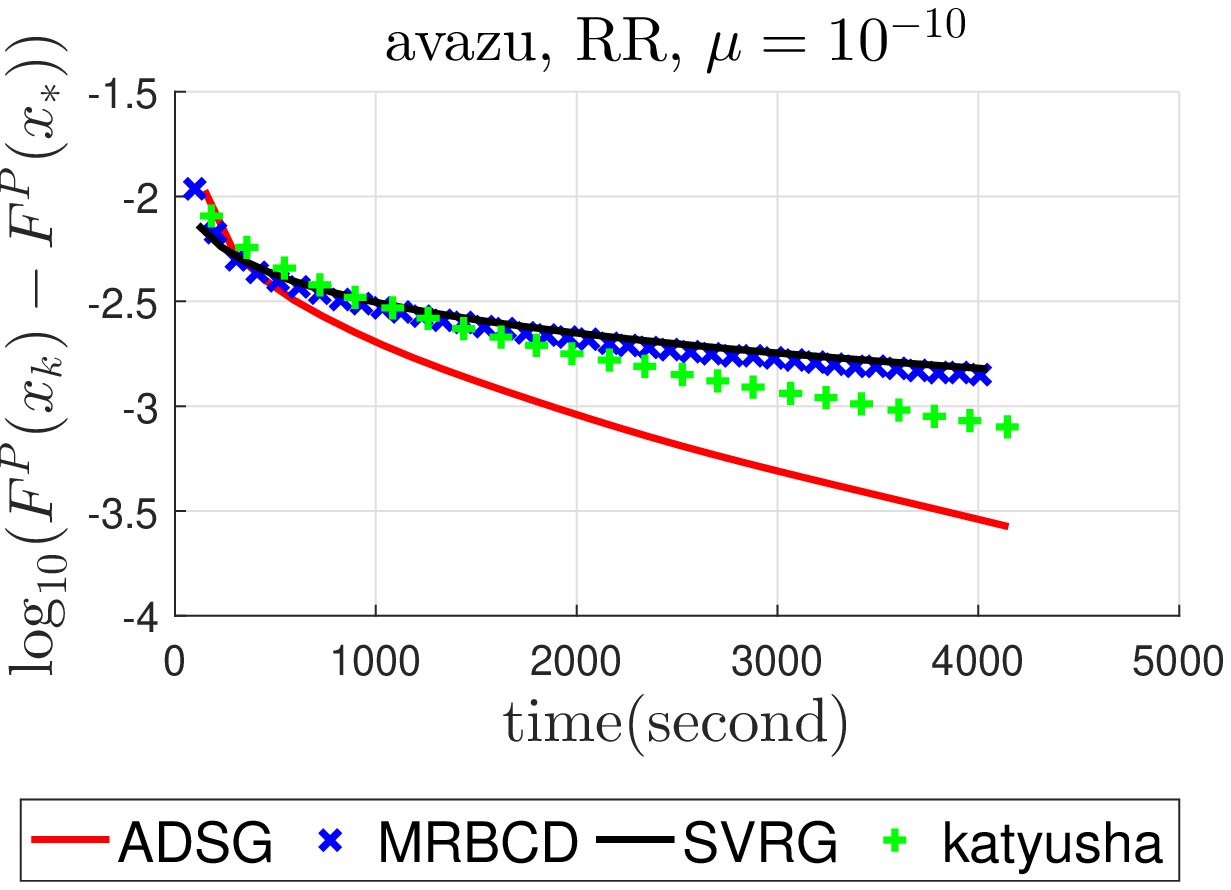}&
			\includegraphics[width = .22\columnwidth]{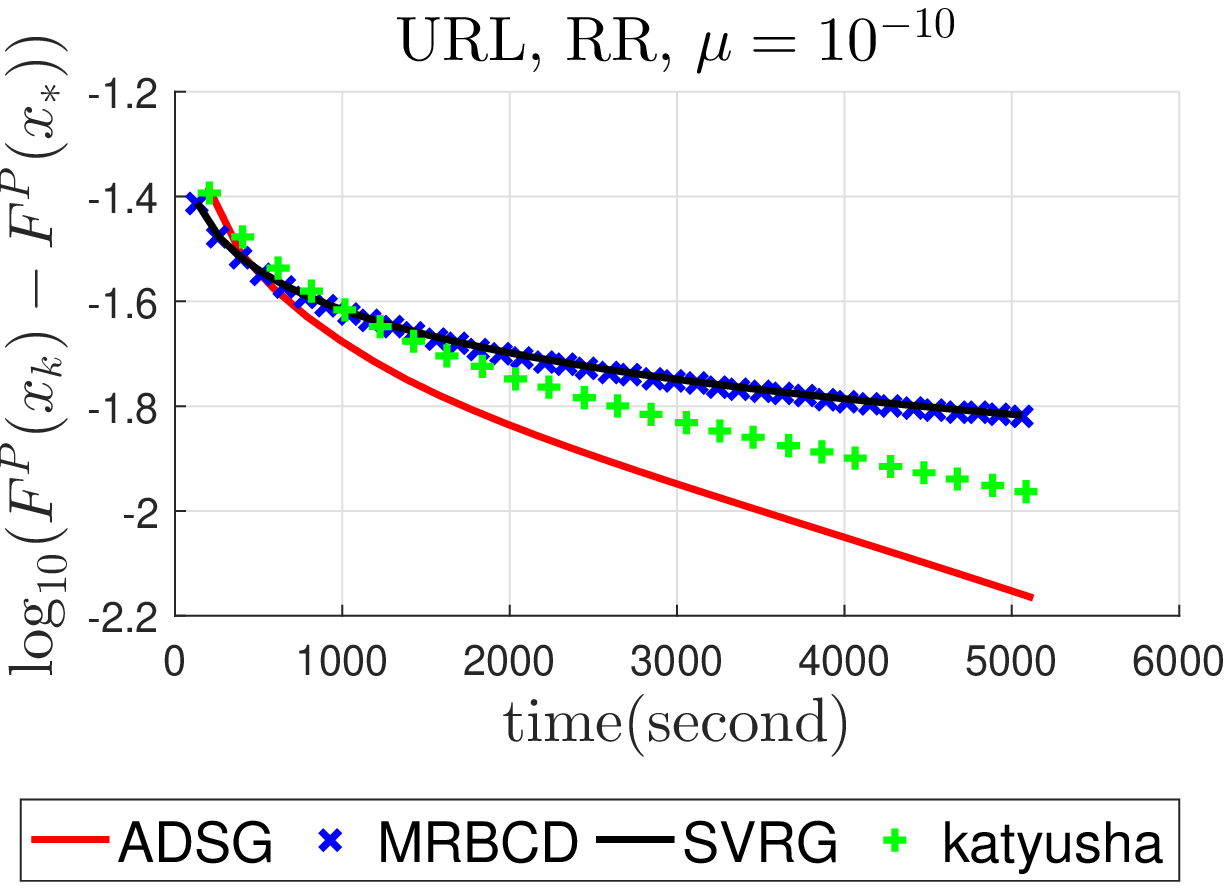}
		\end{tabular}
		\caption{From left to right are the results on news20-binary, kddb-raw, avazu-app, and url-combined.}
		\label{fig_result}
	\end{figure}
	In this section, we conduct several experiments to show the time efficiency of ADSG on huge-scale real problems.\\
	\noindent{\bf Problems} We conduct experiments of several ERM problems using different regularization functions.
	For smooth ERM loss, we use logistic regression and least square regression.
	For non smooth ERM loss, we use SVM and $l_1$ regression.
	We use $\lambda_1\cdot\|\cdot\|_1$, $\lambda_2/2\|\cdot\|_2^2$, and $\lambda_1\cdot\|\cdot\|_1+\lambda_2/2\cdot\|\cdot\|_2^2$ for regularization.
	Four large scale datasets from LibSVM \clr{\cite{CC01a}} are used: kdd2010-raw, avazu-app, new20.binary, and url-combined.
	Their statistics are given in Table \ref{table: statistics}.\\
	\noindent{\bf Algorithms} Katyusha~\clr{\cite{allen2017katyusha}}, MRBCD \clr{\cite{zhao2014accelerated}} (ASBCD has similar performance), and SVRG \clr{\cite{johnson2013accelerating}} are included for comparison.
	MRBCD and ADSG adopts the same block parameter $B$.
	All methods use the same mini batch size $b$.
	We use the default inner loop count described in the original paper for SVRG, MRBCD, and Katyusha.
	We tune the step size to give the best performance.
	For SVRG and MRBCD, it is usually $1/L$.
	\subsection{\bf $l_1$-Logistic Regression}
	In this problem, we set {\small $f_i(\xB) = \log(1+\exp(-y_i\aB_i^\top\xB))$} and set {\small $\PB(\xB) = \lambda\|\xB\|_1$}, where {\small $\{\aB_i, y_i\}_{i=1}^n$} are data points.
	We present the accuracy vs. Evaluated Partial Gradients (EPG) and accuracy vs. time in the first two rows of Figure \ref{fig_result}, along with the parameter $\lambda$ used in experiments.
	The result shows that (i) Katyusha and ADSG have superior convergence rate over non-accelerated SVRG and MRBCD, (ii) ADSG and MRBCD, as doubly stochastic methods, enjoy a better time efficiency than Katyusha and SVRG, and (iii) ADSG has the best performance among all competitors.
	\begin{figure}
		\centering
		\begin{tabular}{c c c}
			\includegraphics[width = .3\columnwidth]{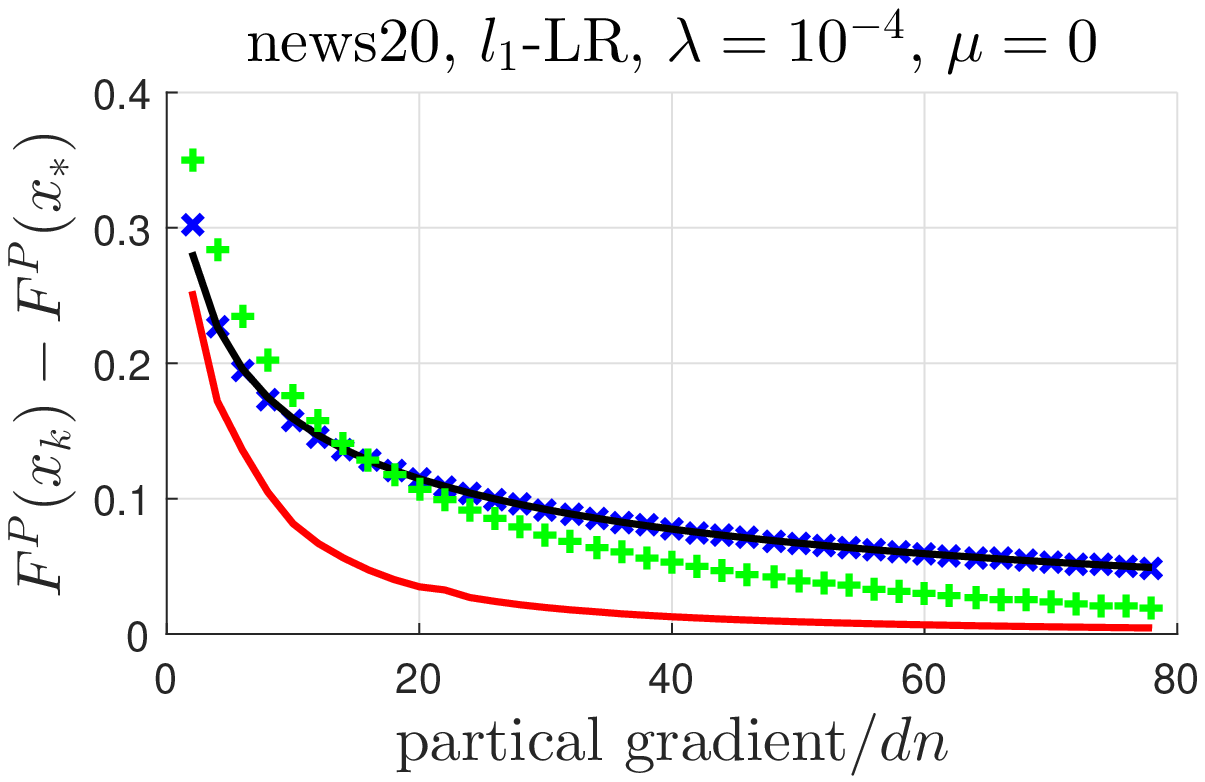}& 
			\includegraphics[width = .3\columnwidth]{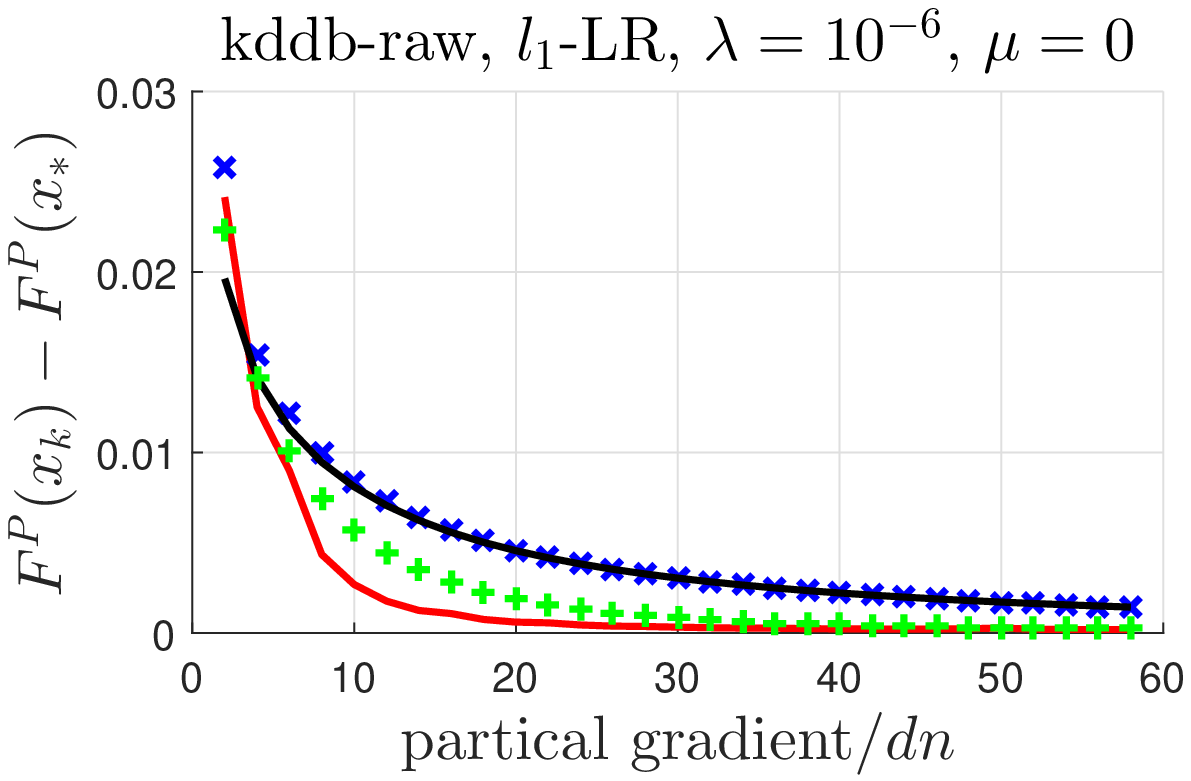}&
			\includegraphics[width = .3\columnwidth]{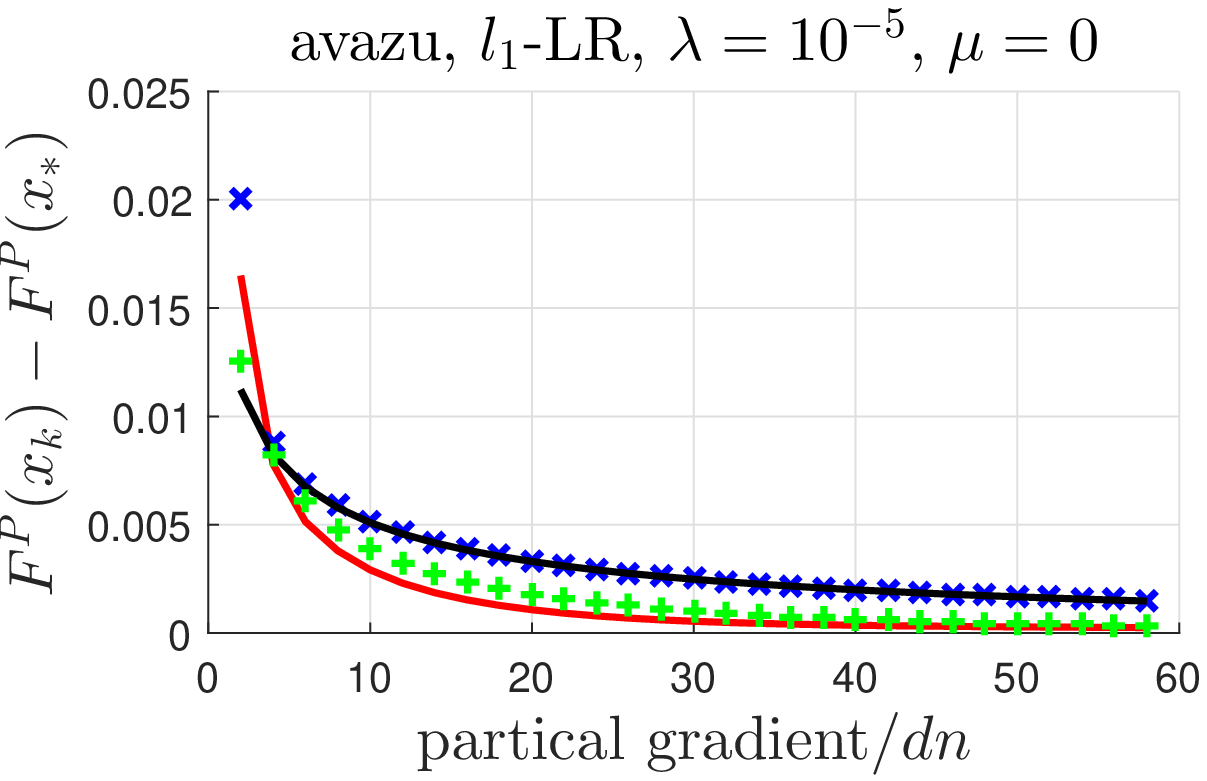}\\
			\includegraphics[width = .3\columnwidth]{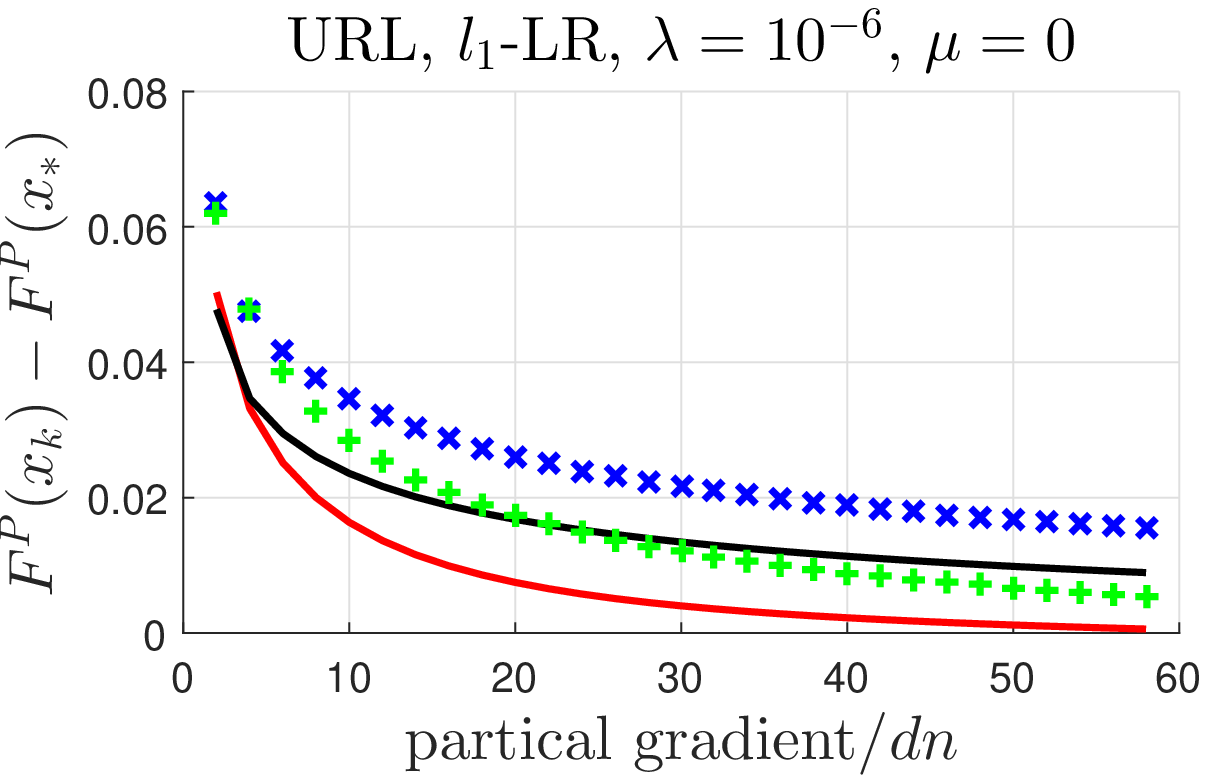}&
			\includegraphics[width = .3\columnwidth]{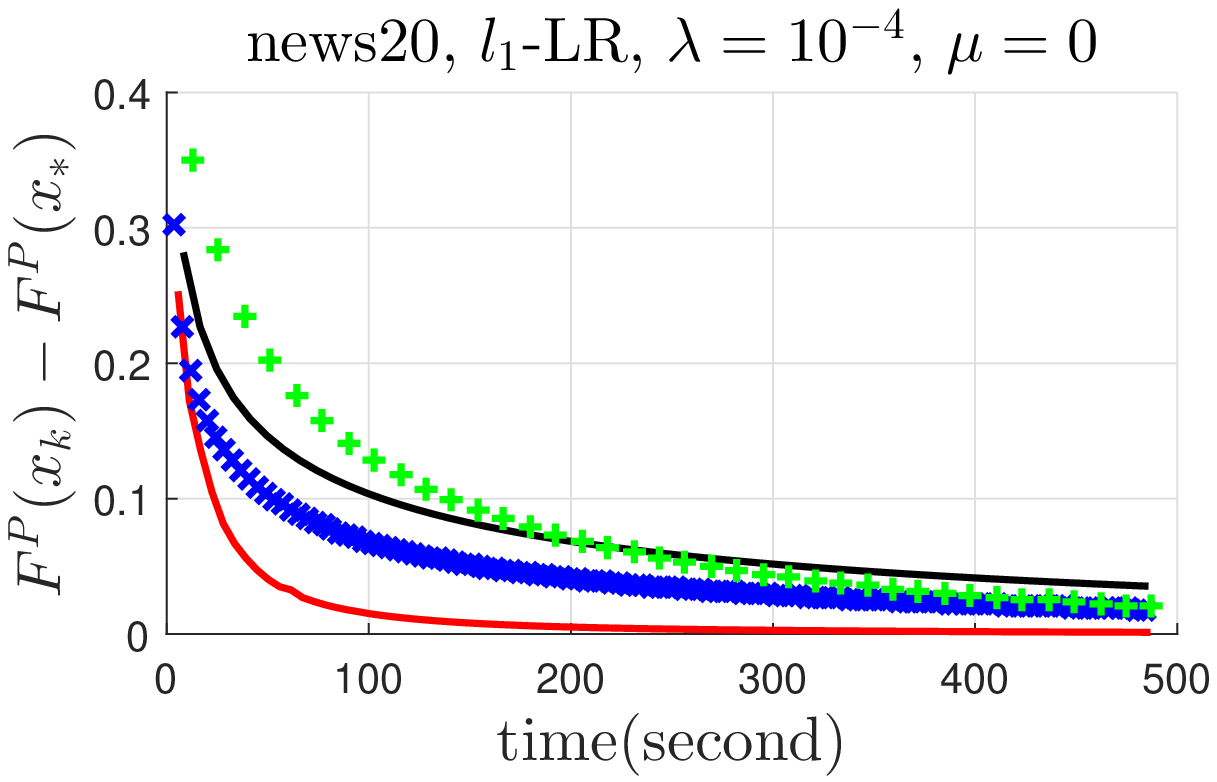} & 
			\includegraphics[width = .3\columnwidth]{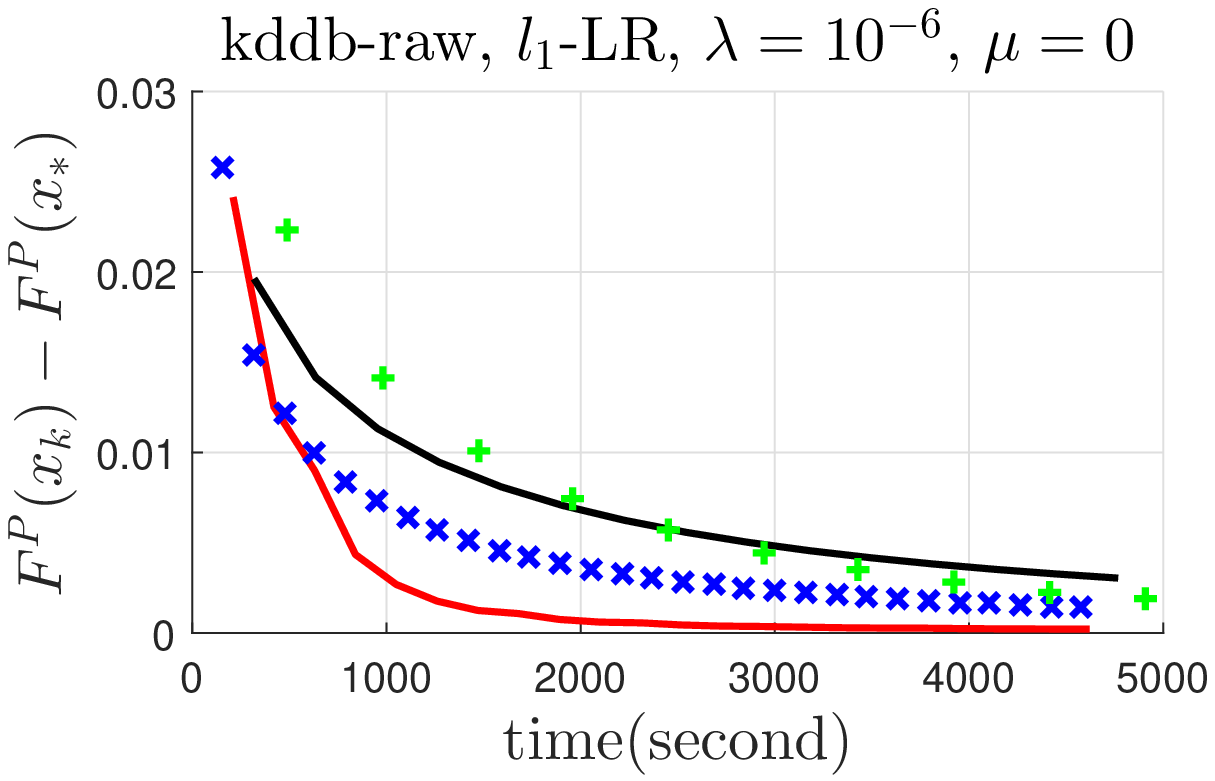} \\ 
			\includegraphics[width = .3\columnwidth]{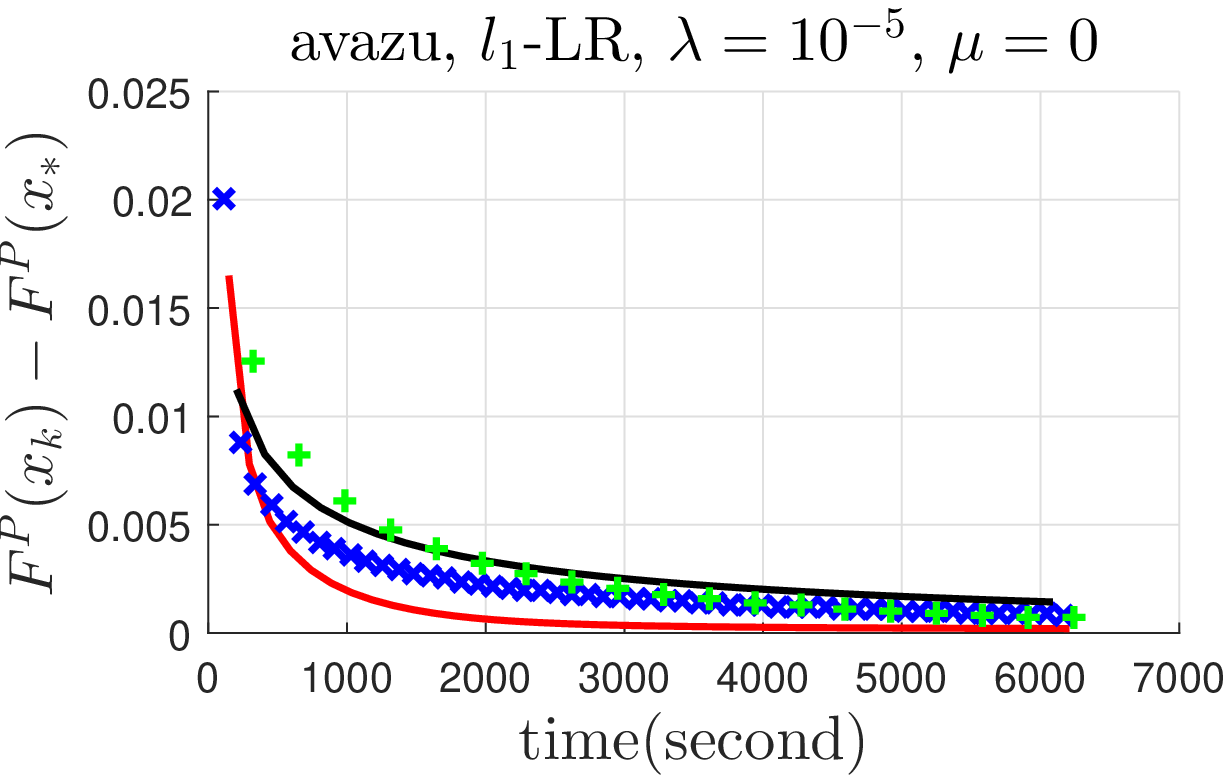} &
			\includegraphics[width = .3\columnwidth]{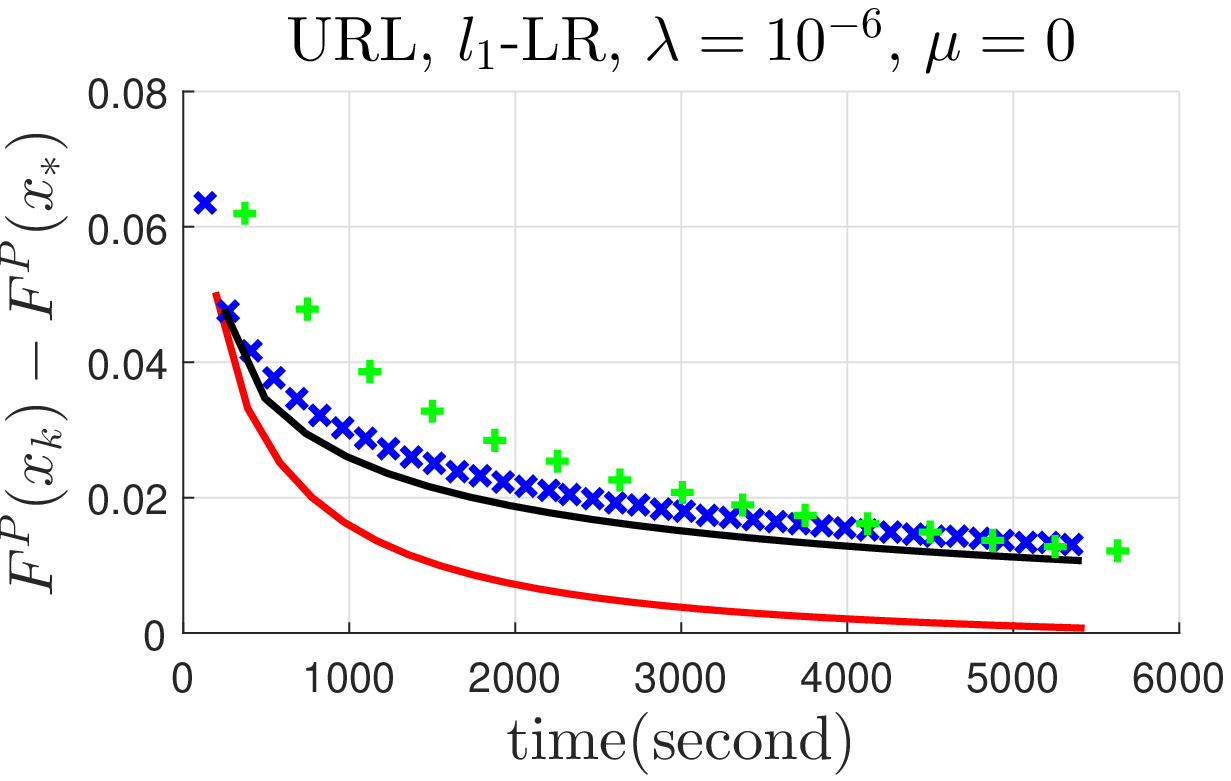}
		\end{tabular}
	\end{figure}
	\subsection{$l_1l_2$-Logistic Regression}
	$f_i(\xB)$ is $\log(1+\exp(-y_i\aB_i^\top\xB))$	and $\PB(\xB)$ is $\frac{\mu}{2}\|\xB\|^2 + \lambda\|\xB\|_1$ in $l_1l_2$-logistic regression.
	The accuracy vs. EPG and accuracy vs. time plots are given in the third and fourth rows of Figure \ref{fig_result}.
	We fix $\mu = 10^{-10}$ and use the same $\lambda$ in the $l_1$-logistic regression.
	We observe similar phenomenon here as in the previous experiment.
	Additionally, all methods converge faster in this case, since the problem is strongly convex.
	\subsection{Ridge Regression}
	We set $f_i(\xB) = \frac{1}{2}\|\aB_i^\top\xB-\yB_i\|^2$ and $\PB(\xB) =\frac{\mu}{2}\|\xB\|^2$ in this experiment.
	In the last two rows of Figure \ref{fig_result}, we give the accuracy vs. EPG and accuracy vs. time plots.
	We fix $\mu = 10^{-10}$ in all datasets.
	ADSG shows exceptional computational efficiency in this experiment.\\

	\subsection{Solving Non-smooth ERM problem}
	We conduct experiments on non-smooth ERM problem with linear predictor in this section.
	Specifically, we assume that $f_i(\xB) = \phi(\langle \aB_i, \xB \rangle)$.
	To use the reduction methods AdaptSmooth and JointAdaptSmoothReg mentioned earlier, we define the auxiliary function
	\begin{equation}
		\phi^{(\lambda)}(z) = \max_\beta z\cdot\beta - \phi^*(\beta) - \frac{\lambda}{2}\beta^2
	\end{equation}
	where $\phi_i^*(\beta) = \max_\alpha \alpha\beta - \phi_i(\alpha)$.
	We consider two popular problem where $\phi^{(\lambda)}$ admits a closed form.
	\subsection{Least Absolute Deviation}
	$f_i(\xB) = \phi_i(\langle a_i, \xB \rangle)$, with $\phi_i(y) = |y-y_i|$.
	$\phi_i^*(\beta) = \beta\cdot y_i + 1_{|\beta|\leq 1}$ and \[\phi_i^{(\lambda)}(\alpha) = \begin{cases}
	\alpha - y_i - \lambda/2, ~ &\alpha > y_i + \lambda, \\
	(\alpha - y_i)^2/2\lambda, ~ &|\alpha - y_i| \leq \lambda, \\
	y_i - \alpha - \lambda/2, ~ &\alpha < y_i - \lambda.
	\end{cases}\]
	Hence we have
	\[\partial \phi_i^{(\lambda)}(\alpha) = \begin{cases}
	1, ~ &\alpha > y_i + \lambda, \\
	(\alpha - y_i)/\lambda, ~ &|\alpha - y_i| \leq \lambda, \\
	-1, ~ &\alpha < y_i - \lambda.
	\end{cases}\]
	
	\subsection{Support Vector Machine}
	$f_i(\xB) = \phi_i(\langle a_i, \xB \rangle)$, with $\phi_i(\alpha) = \max\{0, 1-y_i \alpha\}$.
	$\phi_i^*(\beta) = \beta\cdot y_i + 1_{-1\leq y_i\beta \leq 0}$ and \[\phi_i^{(\lambda)}(\alpha) = \begin{cases}
	0, ~ &y_i\alpha > 1, \\
	(y_i\alpha - 1)^2/2\lambda, ~ &-\lambda + 1\leq y_i\alpha \leq 1, \\
	1 - y_i\alpha - \lambda/2, ~ &y_i\alpha <  - \lambda+1.
	\end{cases}\]
	Hence we have
	\[\partial \phi_i^{(\lambda)}(\alpha) = \begin{cases}
	0, ~ &y_i\alpha > 1, \\
	(\alpha - y_i)/\lambda, ~ &-\lambda + 1\leq y_i\alpha \leq 1, \\
	-y_i, ~ &y_i\alpha <  - \lambda+1.
	\end{cases}\]
	\section{Conclusion}
	An accelerated doubly stochastic algorithm called ADSG is proposed in this paper.
	We give its convergence analyses, and compare our algorithm to the state-of-the-art in large scale ERM problems.
	The result is promising.
	
\bibliographystyle{plainnat}
\bibliography{avrbcd}
\clearpage

\end{document}